\documentclass[sigconf,nonacm]{acmart}

\AtBeginDocument{%
	}

\usepackage[disable]{todonotes}
\usepackage{amsmath,amsfonts,xcolor,graphicx,wrapfig,enumitem} 
\usepackage{thmtools, thm-restate}
\usepackage{algorithmicx,algorithm}
\usepackage[noend]{algpseudocode}
\usepackage{mathtools}
\usepackage{booktabs}
\usepackage{adjustbox}
\usepackage{xspace}
\usepackage{color,colortbl}
\usepackage{marvosym}
\usepackage{stmaryrd}
\usepackage{tikz}
\usetikzlibrary{positioning, arrows.meta}
\usepackage{tcolorbox}
\usepackage{multirow}
\usepackage{algorithm}
\usepackage{algpseudocode}
\usepackage[capitalise]{cleveref}
\usepackage[compatibility=false]{caption}
\usepackage{subcaption}
\usepackage{placeins} 

\usepackage{booktabs, multirow} 
\usepackage{siunitx}[=v2]

\renewcommand{\arraystretch}{1.2}
\usepackage{array}
\newcolumntype{H}{>{\setbox0=\hbox\bgroup}c<{\egroup}@{}}
\usetikzlibrary{shapes,positioning,arrows,calc,automata,matrix,fit,arrows.meta}
\tikzstyle{state}+=[minimum size = 6mm, inner sep=0,outer sep=1]
\tikzset{->,>=stealth'}

\newcommand{\Reach}{\mathsf{Reach}}

\newcommand{\policy}{\sigma}

\newcommand\tool[1]{{\scshape #1}\xspace}
\newcommand{\dtcontrol}{\tool{dtControl}}
\newcommand{\prism}{\tool{Prism}}
\newcommand{\storm}{\tool{Storm}}
\newcommand{\scots}{\tool{Scots}}
\newcommand{\uppaal}{\tool{UppAal}}
\newcommand{\buddy}{\tool{BuDDy}}

\algrenewcommand\algorithmicrequire{\textbf{Input:}}
\algrenewcommand\algorithmicensure{\textbf{Output:}}

\usepackage{pgfplots}
\pgfplotsset{compat=1.18}

\newcommand{\inlineheadingbf}[1]{\smallskip\noindent{\bfseries #1.}}

\newcommand{\cwcomment}[1]{\todo[color=red!20]{\textbf{\textsf{CW:}} #1}}

\newcommand{\ra}{\ensuremath{\rightarrow}\xspace}

\newcommand{\lsem}{\llbracket}
\newcommand{\rsem}{\rrbracket}
\newcommand{\Sem}[1]{\lsem#1\rsem}

\newcommand{\en}[1]{\enskip#1\enskip}

\newcommand{\Bool}{\mathbb{B}}
\newcommand{\Nat}{\mathbb{N}}

\newcommand{\Rat}{\mathbb{Q}}
\newcommand{\Real}{\mathbb{R}}
\newcommand{\Pred}{\mathbb{P}(\Var,\Dom)}

\newcommand{\Dom}{\mathbb{D}}
\newcommand{\power}[1]{2^{#1}}

\newcommand{\true}{\mathtt{true}}

\newcommand{\Act}{\mathit{Act}}
\newcommand{\Var}{\mathit{Var}}
\newcommand{\PVar}{\mathit{PVar}}

\newcommand{\Eval}{\mathit{Eval}(\Var,\Dom)}

\newcommand{\dec}{\ensuremath{\mathtt{succ}}\xspace}

\newcommand{\pdd}{\ensuremath{\mathcal{{P}}}\xspace}
\newcommand{\bdd}{\ensuremath{\mathcal{{B}}}\xspace}

\newcommand*\cir[1]{\tikz[baseline=(char.base),font=\small\bfseries]{%
		\node[shape=circle,draw,inner sep=0.5pt] (char) {{#1}};}}

\newcommand*\dcir[1]{\tikz[baseline=(char.base),font=\small\bfseries\color{white}]{%
		\node[shape=circle,draw,inner sep=0.5pt,fill=black] (char) {{#1}};}}
	
\newcommand*\circled[1]{\kern-2.5em%
	\put(0,3){\color{white}\circle*{11}}\put(0,3){\circle{9}}%
	\put(-2.8,0){\color{black}\bfseries\small#1}~~}

\newcommand*\darkcircled[1]{\kern-2.5em%
	\put(0,3){\color{black}\circle*{9}}\put(0,3){\circle{9}}%
	\put(-2.8,0){\color{white}\bfseries\small#1}~~}

\newcommand{\paragraphh}[1]{\smallskip\noindent\textbf{#1}}

\begin{document}
	\title{Explaining Control Policies through Predicate Decision Diagrams}

\author{Debraj Chakraborty}
\email{chakraborty@fi.muni.cz}
\orcid{0000-0003-0978-4457}
\affiliation{\institution{Masaryk University}
\city{Brno}\country{Czech Republic}}

\author{Clemens Dubslaff}
\email{c.dubslaff@tue.nl}
\orcid{0000-0001-5718-8276}
\affiliation{\institution{Eindhoven University of Technology}
\city{Eindhoven}\country{The Netherlands}}

\author{Sudeep Kanav}
\email{kanav@fi.muni.cz}
\orcid{0000-0001-6078-4175}
\affiliation{\institution{Masaryk University}
\city{Brno}\country{Czech Republic}}

\author{Jan Kretinsky}
\email{jan.kretinsky@tum.de}
\orcid{0000-0002-8122-2881}
\affiliation{\institution{Masaryk University}
\city{Brno}\country{Czech Republic}}
\affiliation{\institution{Technical University of Munich}
\city{Munich}\country{Germany}}

\author{Christoph Weinhuber}
\email{christoph.weinhuber@cs.ox.ac.uk}
\orcid{0000-0002-1600-4933}
\affiliation{\institution{University of Oxford}
\city{Oxford}\country{United Kingdom}}

	\begin{abstract}
		Safety-critical controllers of complex systems are hard to construct manually. Automated approaches such as controller synthesis or learning provide a tempting alternative but usually lack explainability. To this end, learning decision trees (DTs) has been prevalently used towards an interpretable model of the generated controllers.
		However, DTs do not exploit shared decision-making, a key concept exploited in binary decision diagrams (BDDs) to reduce their size and thus improve explainability.
In this work, we introduce \emph{predicate decision diagrams} (PDDs) that extend BDDs with predicates and thus unite the advantages of DTs and BDDs for controller representation.
We establish a synthesis pipeline for efficient construction of PDDs from DTs representing controllers, exploiting reduction techniques for BDDs also for PDDs. 
	\end{abstract}

	
	
	\maketitle

\section{Introduction}\label{sec:intro}

\paragraphh{Controllers} of dynamic systems are hard to construct for a number of reasons: (i) the systems---be it software or cyber-physical systems---are complex due to phenomena such as concurrency, hybrid dynamics, uncertainty, and their steadily increasing sizes; (ii) they are often safety-critical, with a number of complex specifications to adhere to.
Consequently, manually crafting the controllers is error-prone. \emph{Automated controller synthesis} or \emph{learning} provide a tempting alternative, where controllers are constructed algorithmically and their correctness is ensured with human participation limited only to the specification process.
Numerous approaches have been presented to solve this problem, e.g., exploiting game-theoretic methods to construct a winning strategy (a.k.a.\ policy) in a game played by the controller and its environment.

\paragraphh{Explainability} of the automatically constructed controllers (or rather its lack) poses, however, a fundamental obstacle to practical use of automated construction of controllers.
Indeed, an unintelligible controller can hardly be accepted by an engineer, who is supposed to maintain the system, adapt its implementation, provide arguments in the certification process, not to speak of legally binding explainability of learnt controllers \cite{EUX19,USX20,EUX21}.
Unfortunately, most of the synthesis and learning approaches compute controllers in the shape of huge tables of state-action pairs, describing which actions can be played in each state, or of cryptic black boxes.
Such representations are typically so huge that they are utterly incomprehensible.
An explainable representation should be (i) intuitive with decision entities being small enough to be contained in the human working memory, and (ii) in a simple enough formalism so that the decisions can be read off in a way that relates them to the real states of the system.

\paragraphh{Decision trees} (DTs, cf.~\cite{mitchellML}) have a very specific, even unique position among machine-learning models due to their interpretabilty.
As a result, they have become  the predominant formalism for explainable representations of controllers \cite{counterexample-learning,viktor-reactivesynth,viktor-qest19,sos-qest19,neider-dt-invariants,dtcontrol2,DBLP:journals/sttt/JungermannKW23,6408575,gupta2015policy,vos2023optimal} over the past decade.
\Cref{fig:dt_example} shows an example of a policy representation using DTs.
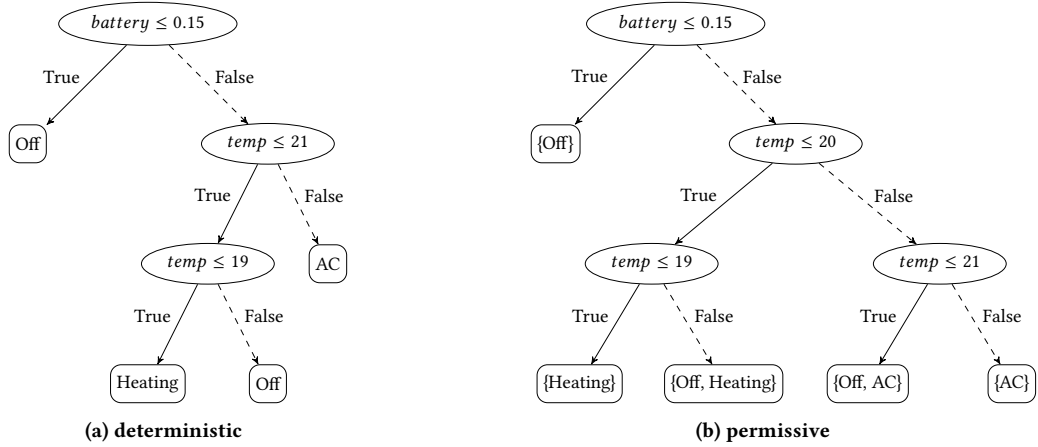
\begin{figure*}[t]
	\centering
	\captionsetup[subfigure]{justification=centering}
	\subfloat[][deterministic]{
		\label{fig:dt_example_deterministic}
		\resizebox{!}{0.3\textwidth}{\tikzstyle{dn} = [draw, ellipse]
\tikzstyle{leaf} = [draw, rounded corners=5pt, minimum height=2em]
\tikzstyle{t} = []
\tikzstyle{f} = [dashed]
\newcommand{\writesub}[1]{\\\small{$\boldsymbol{#1}$}}
\begin{tikzpicture}
  \node[dn] at (2, 0)   (root)  {$battery \le 0.15$};
  \node[dn] at (4, -2) (r)     {$temp \le 21$};
  \node[dn] at (3, -4)  (rl)    {$temp \le 19$};
  
  \node[leaf] at (0, -2)  (l)     {Off};
  \node[leaf] at (5, -4) (rr)    {AC};
  \node[leaf] at (2, -6)  (rll)   {Heating};
  \node[leaf] at (4, -6)  (rlr)  {Off};

  \draw[t] (root) -- node[anchor=east, yshift=4pt]{True}  (l);
  \draw[f] (root) -- node[anchor=west, yshift=4pt]{False} (r);
  \draw[t] (r)    -- node[anchor=east, yshift=4pt]{True}  (rl);
  \draw[f] (r)    -- node[anchor=west, yshift=4pt]{False} (rr);
  \draw[t] (rl)   -- node[anchor=east, yshift=4pt]{True}  (rll);
  \draw[f] (rl)   -- node[anchor=west, yshift=4pt]{False} (rlr);

\end{tikzpicture}}
	}
	\hspace{5em}
	\subfloat[][permissive]{
		\label{fig:dt_example_permissive}
		\resizebox{!}{0.3\textwidth}{\tikzstyle{dn} = [draw, ellipse]
\tikzstyle{leaf} = [draw, rounded corners=5pt, minimum height=2em]
\tikzstyle{t} = []
\tikzstyle{f} = [dashed]
\tikzstyle{nf} = [anchor=west, yshift=4pt]
\tikzstyle{nt} = [anchor=east, yshift=4pt]
\newcommand{\writesub}[1]{\\\small{$\boldsymbol{#1}$}}
\begin{tikzpicture}
  \node[dn] at (2, 0)   (root)  {$battery \le 0.15$};
  \node[dn] at (4, -2) (r)     {$temp \le 20$};
  \node[dn] at (1.6, -4)  (rl)    {$temp \le 19$};
  \node[dn] at (6.4, -4) (rr)    {$temp \le 21$};
  
  \node[leaf] at (0, -2)  (l)     {\{Off\}};
  \node[leaf] at (0.4, -6)  (rll)   {\{Heating\}};
  \node[leaf] at (2.8, -6)  (rlr)   {\{Off, Heating\}};
  \node[leaf] at (5.2, -6) (rrl)   {\{Off, AC\}};
  \node[leaf] at (7.6, -6) (rrr)   {\{AC\}};

  \draw[t] (root) -- node[nt]{True}  (l);
  \draw[f] (root) -- node[nf]{False} (r);
  \draw[t] (r)    -- node[nt]{True}  (rl);
  \draw[f] (r)    -- node[nf]{False} (rr);
  \draw[t] (rl)   -- node[nt]{True}  (rll);
  \draw[f] (rl)   -- node[nf]{False} (rlr);
  \draw[t] (rr)   -- node[nt]{True}  (rrl);
  \draw[f] (rr)   -- node[nf]{False} (rrr);

\end{tikzpicture}}
	}
	\caption{Examples of DT control policy representations: deterministic control policy (left) and  permissive control policy with multiple actions at some states (right).}
	\label{fig:dt_example}
\end{figure*}
	
In contrast, reduced ordered \emph{binary decision diagrams (BDDs)}~\cite{BDD} are very popular to represent large objects compactly, 
exploiting bit-wise representation with automatic reduction mechanisms. 
However, they were barely used as policy representation of controllers due to several advantages of DTs:
\begin{enumerate}[label=\protect\darkcircled{\arabic*},leftmargin=1.5em]
	\item To represent the controller with non-Boolean state variables using a BDD, the variables are usually \emph{bit-blasted} into a binary encoding. 
	The original controller is then rewritten in terms of these Boolean variables, and a BDD is constructed over them.
	In contrast, DTs can use more general predicates when arriving at a decision.
	This allows DTs to be smaller, as a single predicate can encompass multiple bits. 
	Additionally, the use of customized predicates makes DTs easier to understand, as they can better incorporate domain knowledge.
	\item The order of predicates along paths in DTs is not fixed as in reduced ordered BDDs. This gives more flexibility and can lead to substantial space savings in DTs.
	\item DT learning allows ``don't care'' values to be set in any way so that the learnt DT is the smallest. In contrast, BDDs natively represent only fully specified sets, with no heuristic how to resolve the ``don't cares'' resulting typically in representation larger than necessary.
\end{enumerate}
Still, reduced ordered BDDs exhibit reductions not present in DTs, notably merging isomorphic subdiagrams.
This would also be beneficial for explainability of controllers: subdiagram merging
(i) reduces the representation size and hence improves explainability following Occam’s razor, according to which the best explanation is the smallest among all possible ones~\cite{Goo77}, and 
(ii) the psychological complexity towards understanding is also reduced, similar to avoiding code duplication in software~\cite{zuse2019software}.

In this paper, we revisit the challenge of marrying concepts from DTs and reduced ordered BDDs towards controller explanations that exhibit advantages of both representations.
For this, we build upon the idea of linear decision diagrams~\cite{LinearDD} and introduce \emph{predicate decision diagrams (PDDs)} as multi-terminal BDDs~\cite{MTBDD} over bit representations of predicates that inherit all BDD advantages but avoid bit-blasting. PDDs extend DTs and maintain their predicates, profiting from their interpretability, but allow for reduction methods as possible in ordered BDDs.

Specifically, we propose the following approach towards generating reduced ordered PDD from DTs learnt for control policies:
\begin{enumerate}[label=\protect\circled{\arabic*},leftmargin=1.5em]
	\item We collect the predicates mined by the DT representation of a given controller and use them as ``variables'' in reduced ordered BDDs. 
	This enables use of BDD reductions and mature BDD tooling for PDDs, eliminating issue~\dcir{1}.
	\item We apply variable reordering techniques~\cite{Rud93} to reduce the size of the PDD, alleviating issue~\dcir2.
	\item We use Coudert and Madre's method~\cite{CouMad90} restricting the BDD support to ``care sets'', compressing PDDs further, fixing issue~\dcir3.
\end{enumerate}
Notably, due to $\cir{1}$ and $\cir{2}$, contradictory branches can arise in PDDs, e.g.\ by observing contradictory predicates $(\mathit{temp}\leq 19)$ and $\neg(\mathit{temp} \leq 20)$ on one decision path.
We present a detection and elimination mechanism to remove contradictory branches, leading to \emph{consistent PDDs}, further reducing their size and hence improving explainability of PDDs. 
We then use our PDD synthesis pipeline, to address the following two research questions:
\begin{enumerate}[label=\textbf{RQ\textsubscript{\arabic*}},leftmargin=*,nosep]
	\item \emph{How do the sizes of PDDs and bit-blasted BDDs relate?}\label{rq1b}\\
		Following our construction method, PDDs can be represented as (multi-terminal) BDDs, where, however, predicate evaluation does not use bit-blasting as common in classical controller representation via BDDs. The question is whether avoiding bit-blasting can yield more concise representations.
    \item \emph{Can controllers be represented more concisely by using ordered PDDs rather than by using DTs?}\label{rq1a}\\
    	Ordered PDDs impose an order on predicates, while DTs can put important decisions earlier in separate branches, leading to smaller tree sizes. However, PDDs allow for merging subdiagrams.
	The question is which of the complementary reduction mechanisms is more effective.
	\item \emph{What is the impact of consistency, reordering, and care-set reduction onto PDD explanations?}\label{rq2}\\
		We consider several techniques to reduce the size of PDDs and hence improve their explainability. This research question asks for an ablation study on the impact of these techniques.
\end{enumerate}

\paragraphh{Our contributions} answer the research questions above and can be summarized as follows:
\begin{enumerate}[label=$\bullet$,leftmargin=*]
	\item We introduce \emph{PDDs}, a BDD data structure allowing for both explainable and small representation of controllers.
	\item We design an algorithm effectively synthesizing PDD representations, in particular profiting from DT learning and using reasoning modulo DT theory to cater for branch consistency.
	\item We experimentally show that PDDs and several reduction techniques are beneficial for explaining control policies compared to DT and bit-blasted BDD representations. 
\end{enumerate}
We experimentally demonstrate that we almost close the gap between bit-blasted BDD and DT representation sizes. In half of the cases PDDs match the size of DTs or even reduce the size, while retaining the same level of explainability as DTs.
	Quantitatively, we reduce the gap between BDD and DT sizes on average by 88\% on the standard benchmarks.
In conclusion, 
	we present PDDs as a viable alternative to DTs, that can offer the best of the both worlds: it is as effective as the DTs as an explainable representation, and can take advantage of the traditional BDD-based tools.
	Proofs of lemmas and theorems that did not fit into the main paper can be found in the appendix.
\section{Related Work}
\paragraphh{Decision trees (DTs)} and similar decision diagram formalisms are prevalently used in explainable AI~\cite{adadi2018peeking} 
to create interpretable representations~\cite{torfah2021synthesizing,verwer2019learning,verhaeghe2020learning} of black-box machine learning models.
They have been suggested for (approximately) representing controllers and counterexamples in probabilistic systems in \cite{counterexample-learning}. The ideas have been extended to other settings, such as reactive synthesis \cite{viktor-reactivesynth} or hybrid systems \cite{sos-qest19}. 
DTs have been used to represent and learn strategies for safety objectives in \cite{neider-fmcad19} and to learn program invariants in \cite{neider-dt-invariants}.
Further, DTs were used for representing the strategies during the model checking process, 
namely in strategy iteration~\cite{DBLP:conf/ijcai/BoutilierDG95} or in simulation-based algorithms \cite{PnH}.
The tool dtControl implements automated synthesis of DTs from controllers~\cite{dtcontrol}.

\paragraphh{Binary decision diagrams (BDDs)}~\cite{BDD} are concisely representing Boolean functions,
widely used in symbolic verification~\cite{McM93} and circuit design~\cite{Min96}.
In the context of controllers, they have been used to represent planning strategies~\cite{DBLP:conf/aaai/CimattiRT98},
hybrid system controllers~\cite{scots,pessoa}, to compress numerical controllers~\cite{DellaPenna2009},
as well as for small representations of controllers through determinization~\cite{DBLP:conf/adhs/ZapreevVM18}.
\emph{Multi-terminal decision diagrams (MTBDDs,~\cite{MTBDD})}, or more generally referred to as
\emph{algebraic DDs (ADDs,~\cite{ADD})}, have been used in reinforcement learning for synthesizing 
policies~\cite{DBLP:conf/uai/HoeySHB99,DBLP:conf/nips/St-AubinHB00}.
MTBDDs extend BDDs with the possibility to have multiple values in the terminal nodes.
\emph{Multi-valued DDs}~\cite{MDD} are extending DDs to allow for multiple decision outcomes in a node
over a finite domain rather than only two values.
While concise, BDDs are usually not as well-suited for explaining controllers, since states and decisions 
are usually using standard bit-vector representations that are difficult to map to actual state values and decisions.

\paragraphh{Extensions of Decision Diagrams.} To avoid bit-blasting and directly enable reasoning on numerical values in decision diagrams,
several extensions of BDDs have been proposed. Notably, \emph{edge-valued BDDs}~\cite{NDD} replace bit-level
aggregation by summation over integer-labeled edges. 
Following ideas from satisfiability solving modulo theories (SMT,~\cite{BarSebSes21}),
the use of BDDs with different theories has been considered in many areas, e.g.\ to directly speed 
up predicate abstraction and SMT~\cite{CavCimFra07}.\todo{boolean abstraction/purification}
Closely related to our work are specific extensions of BDDs where nodes are labeled by predicates from a suitable theory.
Here, notable examples are \emph{equational DDs}~\cite{FriPol00} where nodes might contain equivalences of the 
form $x=y$, \emph{difference DDs}~\cite{DifferenceDD} with predicates of the form $x-y>c$, mainly used
in the verification of timed systems, \emph{constrained DDs}~\cite{ConstrainedDD} extending 
multi-valued DDs with linear constraints to solve constraint satisfaction problems, and 
\emph{linear DDs}~\cite{LinearDD} where nodes are labeled with linear arithmetic predicates.
The latter are the closest to our approach towards PDDs, capturing linear predicates and corresponding to \emph{oblique DTs}. The subclass of \emph{axis-aligned DTs} then reflects in our notion of \emph{axis-aligned PDDs}, having the form
$x\sim c$ where $\sim$ being a binary comparison relation. 
Different to our PDDs, linear DDs enforce an implication-consistent ordering on the 
predicates and use complemented edges to enable constant-time negation operations.
We do not opt for complemented edges in PDDs, since they render DDs much more difficult to comprehend.
The phenomenon of inconsistent paths in BDDs over predicates from a theory has been already
observed in the literature, including sophisticated methods to resolve inconsistencies~\cite{FonGri02,TinHar96,FriPol00,DifferenceDD,LinearDD}.
Since we aim at explainability of controllers, we focus on simplistic predicates and thus can establish a direct yet effective method of ensuring consistency.

\smallskip
\noindent Dubslaff et al.~\cite{DubKloPas24} present \emph{template DDs} (TDDs) to model and improve explainability of self-adaptive systems and DTs learnt from controllers.
Following the concept of procedures, templates in TDDs can be instantiated with decision-making patterns, leading to hierarchical (unordered) multi-valued DDs.
Similar to our approach, DT predicates are interpreted as Boolean variables and sharing of decision making yields size and cognitive load reductions, both improving explainability.

\smallskip
\noindent Recently, there has been research on learning DDs for classification of training data.
Cabodi et al.~\cite{DATE-Cabodi21} propose a SAT-based approach towards a BDD representation of a DT.
Hu et al.~\cite{AAAI-Hu22} use SAT- and a lifted MaxSAT-based model to learn optimal BDDs with bounded depth.
Florio et al.~\cite{OptimalDDs_Classification} propose a mixed-integer linear programming method to learn optimal (not necessarily binary) DDs.


\section{Preliminaries}\label{sec:prelim}

For a set $X$, we denote its power set by $\power{X}$.
A \emph{total order} over a finite set $X$ is a bijection $\tau\colon X \ra \{1, \ldots, |X|\}$ where $|X|$ is the number of elements in $X$.
Let us fix a set of \emph{variables} $\Var$ and a \emph{variable domain} $\Dom$.
Canonical domain candidates are the boolean domain $\Bool=\{0,1\}$, natural numbers $\Nat$, rational 
numbers $\Rat$, and real numbers $\Real$.
An \emph{evaluation} is a function $\epsilon\colon \Var\ra\Dom$; the set of evaluations is denoted by $\Eval$. 
An (axis-aligned) \emph{predicate} is of the form $x\sim c$ where $x\in\Var$ is a variable,
$c\in\Dom$, and $\mathbin{\sim}\in\{=,\geqslant,>\}$ is a comparison relation over $\Dom$. 
The set of predicates is denoted by $\Pred$.
A \emph{simple} predicate is of the form $x>0$ with $x\in\Var$, denoted by $x$ for brevity\footnote{ 
Therefore, we consider variables to be contained in the set of predicates, i.e., $\Var\subseteq\Pred$,
interpreted as simple predicates.}. 
A predicate $x\sim c\in\Pred$ is \emph{satisfied} in an evaluation $\epsilon\colon\Var\ra\Dom$, 
denoted $\epsilon\models (x\sim c)$, iff $\epsilon(x) \sim c$.
We introduce propositional formulas over $\Pred$ by the grammar $\phi ::= \true \mid p \mid \neg\phi \mid \phi \wedge \phi$
where $p$ ranges over $\Pred$. The satisfaction relation $\models$ extends towards 
formulas by $\epsilon\models\neg\phi$ iff $\epsilon~\not\models~\phi$, and
$\epsilon\models\phi_1\land\phi_2$ iff $\epsilon\models\phi_1$ and $\epsilon\models\phi_2$.

\paragraphh{Bit-Blasting.}
Let $S$ be a finite set of states with $|S|$ elements and $n = \lceil \log_2 |S| \rceil$. 
We define bit-blasting as an injective mapping $\beta \colon S \to \Bool^n$, which is constructed in two steps: First, we enumerate states by a total order $\tau\colon S \ra \{1,\ldots,|S|\}$. Second, for each $s\in S$, we assign the binary representation $\beta(s) = (b_{i,1}, b_{i,2}, \dots, b_{i,n})$ where $(b_{i,1}, \ldots, b_{i,n})$ is the $n$-bit encoding of the integer $i=\tau(s)-1$.
Bit-blasting converts each high-level state into a unique bit vector.

\subsection{Policies}
\begin{definition}[Policy]
	For an operational model with a set of states $S$ and actions $\Act$, a \emph{policy} $\sigma\colon S \rightarrow 2^{\Act}$ selects for every state $s\in S$ a set of actions $\sigma(s) \subseteq \Act$.
\end{definition}
	Note that this definition allows \emph{permissive} 
	(or \emph{nondeterministic}) 
policies that can provide multiple possible actions for a state.
A policy $\sigma$ is \emph{deterministic} if $|\sigma(s)|=1$ for all $s\in S$.

	We assume that the set of states $ S$  is structured, i.e., every 
	state is a tuple of values of state variables. 
	In other words, states are not monolithic (e.g., a simple numbering) but there are multiple factors defining it.
	Each of these factors is defined by a variable from the set of variables $\Var$.
	Therefore, we can alternatively view a policy as 
	a (partial) function $\sigma\colon \Eval \ra \power{\Act}$ 
	that maps an evaluation $\epsilon\colon \Var \ra \Dom$ to a set of actions $\sigma(\epsilon)\subseteq\Act$.
	
	\begin{example}\label{ex:policy}
		Consider the policy in \Cref{fig:ex-table}. There are two state variables $\Var=\{x,y\}$. The set of actions is $\Act=\{a,b\}$. 
	\end{example}
	Note that a policy is usually partial, i.e., not all evaluations are in $ S$, e.g., standing for states not reachable in the system model.

\subsection{Decision Trees}
\begin{definition}[DT]
	A \emph{decision tree (DT)} $T$ is defined as follows:
\begin{enumerate}[label=$\bullet$,leftmargin=*]
		\item $T$ is a rooted full binary tree, meaning every node either is an inner node
		 and has exactly two children, or is a 
		 leaf node and has no children.
		\item Every inner node $v$ is associated with a predicate in $\Pred$.
		\item Every leaf $\ell$ is associated with an output label $a_\ell$.
	\end{enumerate}
\end{definition}

Numerous methods exist for learning DTs, e.g., CART \cite{cart}, ID3 \cite{id3}, and C4.5 \cite{c45}. 
These algorithms all evaluate different 
predicates by calculating some \emph{impurity measure} and then greedily pick the most 
promising to split the dataset on that predicate into two halves.
This is recursively repeated until a single action in the dataset is reached, constituting a leaf of the resulting binary tree.

A DT represents a function as follows: an
evaluation $\epsilon$ is evaluated by starting at the root of $T$ and traversing the tree until we reach a leaf node $\ell$. 
Then, the label of the leaf $a_\ell$ is our prediction for the input $\epsilon$. 
When traversing the tree, at each inner node $v$ we decide at which child to continue by evaluating the predicate $\alpha_v$ with $\epsilon$. 
If the predicate evaluates to true, we pick the left child, otherwise, we pick the right one. 

When we represent a (permissive) policy $\sigma\colon  S\ra \power{\Act}$ with a DT, 
our input data is the set of states $ S$ and an output labels describe a subset of actions $\Act$.
Unlike algorithms in machine learning, our objective is to learn a DT that represents the policy on all
data points available from the controller.
Therefore, we do not stop the learning based on a stopping criterion 
but overfit to completely capture the training data.
\emph{Safe early stopping}~\cite{dtcontrol2} can be used to produce smaller DTs
 that do not represent the full permissive policy $\sigma$ 
 but 
 a deterministic $\sigma'$ such that $\sigma'(s)\subseteq\sigma(s)$ for all $s\in S$.
	\Cref{fig:ex-tree} depicts a DT representing the policy of \Cref{ex:policy}. 
\begin{figure*}[t]
	\captionsetup[subfigure]{justification=centering}
	\centering
	\newsavebox{\tempbox}
	\savebox{\tempbox}{
		{
			\renewcommand{\arraystretch}{1.5}
			\begin{tabular}[b]{r@{\hspace{1.2em}}r@{\hspace{1.2em}}r@{\hspace{1.5em}}c} \toprule
				$x$ & $y$ & \textsf{actions} \\ \midrule
				0 & 0 & $\{a,b\}$ \\
				1 & 0 & $\{b\}$ \\
				2 & 0 & $\{b\}$ \\
				3 & 1 & $\{a\}$ \\ \bottomrule
			\end{tabular}
		}
	}%
	\subfloat[]{\label{fig:ex-table}
		\usebox{\tempbox}
		
	}\hspace{5em}
	\subfloat[]{\label{fig:ex-tree}
		\resizebox{!}{0.18\textwidth}{\tikzstyle{dn} = [draw, ellipse]
\tikzstyle{leaf} = [draw, rounded corners=5pt, minimum height=2em]
\tikzstyle{t} = []
\tikzstyle{f} = [dashed]
\begin{tikzpicture}
  \node[dn] at (4,1.5)   (root)  {$x > 0$};
  \node[dn] at (2,0) (r)     {$x > 2$};
  
  \node[leaf] at (0, -1.5)  (a)     {$\{a\}$};
  \node[leaf] at (4, -1.5) (b)    {$\{b\}$};
  \node[leaf] at (6, 0)  (aa)   {$\{a,b\}$};

  \draw[t] (root) -- node[anchor=east, yshift=4pt]{True}  (r);
  \draw[f] (root) -- node[anchor=west, yshift=4pt]{False} (aa);
  \draw[t] (r)    -- node[anchor=east, yshift=4pt]{True}  (a);
  \draw[f] (r)    -- node[anchor=west, yshift=4pt]{False}  (b);
\end{tikzpicture}}
	}%
	\caption{An example of a policy in the form of a lookup table (left), and the corresponding decision tree (right). 
	}
	\label{fig:exampleDT}
\end{figure*}
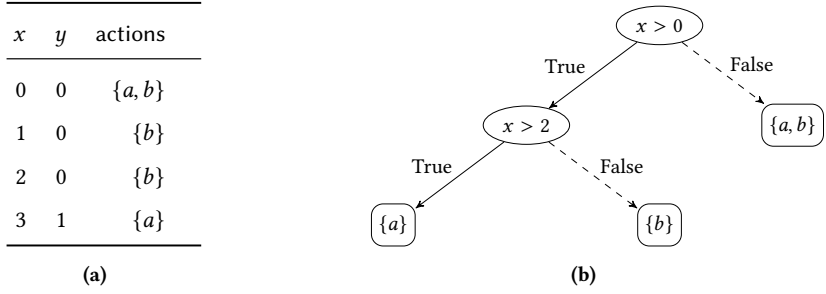

\subsection{Binary Decision Diagrams}

\begin{definition}[BDD]\label{def:bdd}
	A (reduced ordered multi-terminal) \emph{binary decision diagram (BDD)}
	is a tuple $\bdd=(N,\Var,\Act,\lambda,\pi,\dec,r)$ where $N$ is a finite set of nodes,
	$\lambda\colon N \ra \Var\cup\power{\Act}$ is a function labeling decision nodes
	$D\subseteq N$ by Boolean variables $\Var$ and action nodes $N{\setminus}D$
	by sets of actions $\Act$, $\pi$ a total order over $\Var$,
	$\dec\colon D\times\Bool \ra N$ 	is a decision function, and $r\in N$ is a root node such that $\bdd$ is
\begin{enumerate}[label=$\bullet$,leftmargin=*]
		\item \emph{ordered}, i.e., for all $d\in D$ and $b\in\Bool$ with $\dec(d,b)\in D$
			we have $\pi\big(\lambda(d)\big) < \pi\big(\lambda(\dec(d,b))\big)$, and
	\item \emph{reduced}, i.e., all decisions $d, d' \in D$ are \emph{essential} 
		($\dec(d,0)\neq\dec(d,1)$) and $\bdd$ does not contain isomorphic subdiagrams. 
	\end{enumerate}
\end{definition}
A $\pi$-BDD is a BDD that is ordered w.r.t. the variable order $\pi$. 
A \emph{path} in $\bdd$ is an alternating sequence $d_0,b_0,d_1,\ldots,d_k \in (D\times \Bool)^\ast\times N$ 
where $d_{i+1} = \dec(d_i,b_i)$ for all $i<k$.
Given a node $n\in N$, we denote by $\Reach(n)\subseteq N$ the nodes reachable by a path from $n$.
In the following, we assume all nodes $N$ in $\bdd$ to be reachable from root $r$.
By $\bdd_n$ we denote the sub-BDD of $\bdd$ that comprises nodes $\Reach(n)$ and root $n$.
Note that reducedness ensures that for each subset of actions $A\subseteq\Act$ there is at most one
node $a\in N{\setminus}D$ that is labeled by $A$, i.e., $\lambda(a)=A$.

\paragraphh{Semantics of BDDs.} A BDD over $\Var$ and $\Act$ represents a function 
$f\colon \mathit{Eval}(\Var,\Bool)\ra \power{\Act}$ as follows:
every evaluation $\epsilon$ is evaluated by starting at the root $r$ and traversing the graph 
until we reach a terminal action node $v_t$. Then $f(\epsilon)=\lambda(v_t)$. 
When traversing, at each decision node $d\in V$ we decide at which child to continue by checking the value of $\epsilon(\lambda(d))$: 
if $\epsilon(\lambda(d))=1$, we pick $\dec(d,1)$ as the next node.
Otherwise, we pick $\dec(d,0)$ as the next node. 
Formally, given a node $n$ in $\bdd$, we define the function $f=\Sem{\bdd_n}$ 
recursively through evaluations $\epsilon\colon\Var\ra\Bool$:
\begin{center}$
	\Sem{\bdd_{n}}(\epsilon) = 
	\begin{cases}
		\lambda(n) & \text{if } n \in N{\setminus}D\\
		\Sem{\bdd_{\dec(n,\epsilon(\lambda(n)))}}(\epsilon) & \text{if } n\in D\\
	\end{cases}
$\end{center}
To represent a policy $\sigma\colon S\ra\power{\Act}$ using a BDD, 
we treat it as a function $f_{\sigma}$ over evaluations as above:
we represent each state $s\in S$ as a bitvector $\overline{s}$ and
define $f_{\sigma}$ by  $f_{\sigma}(\bar{s})=\sigma(s)$.

While BDDs are a useful data structure to represent policies, 
they have a few disadvantages compared to the DTs. 
Firstly, nodes in BDDs correspond to simple predicates in DTs,
making them less human-interpretable compared to DTs. 
In contrast, DTs supports a richer variety of predicates and variable domains (even exceeding the
class of axis-aligned predicates we focus on in this paper). 
Secondly, BDDs have a canonical form that ensures a unique representation as they are ordered and reduced. 
This facilitates efficient comparison and manipulation of BDDs.
In contrast, DTs allow flexible ordering, which might allow for smaller lengths of decision sequences,
possibly easier to understand.
Thirdly, a DT can be generated from a partial policy, ignoring some inputs
towards smaller representations.
In contrast, existing BDD constructions are not directly optimized for small representations of partial functions.

\section{Predicate Decision Diagrams}\label{sec:MTPDDs}

In this section, we propose \emph{predicate decision diagrams (PDDs)} as an 
instance of linear DDs~\cite{LinearDD} that are extended to represent control policies.
PDDs combine concepts from DTs and BDDs to benefit from both formalism's advantages.

\begin{definition}
	A \emph{predicate decision diagram (PDD)} over $\Var$ and $\Dom$
	is a tuple $\pdd=(N,\Var,\Act,\lambda,\dec,r)$ where $N$ is a finite set of nodes,
	$\lambda\colon N \ra \Pred\cup\power{\Act}$ is a function that labels decision nodes
	$D\subseteq N$ by predicates over $\Var$ and $\Dom$ and action nodes $N{\setminus}D$
	by sets of actions $\Act$, $\dec\colon D\times\Bool \ra N$ 
	is a decision function, and $r\in N$ is a root node.
\end{definition}
Notice the similarities with the BDD definition (see \Cref{def:bdd}).
While in BDDs decision nodes are labeled by variables, they are labeled by predicates in PDDs.
Throughout this section, let us fix a PDD $\pdd$ as above. We inherit BDD notations as expected and call $\pdd$
\begin{enumerate}[label=$\bullet$,leftmargin=*]
	\item \emph{deterministic} if all reachable action nodes are labeled by singletons, i.e., 
		$a\in N{\setminus}D$ implies $|\lambda(a)|=1$.	
	\item \emph{ordered} for a total order $\pi$ over $\lambda(D)$ if
		for all $d\in D$ and $b\in\Bool$ with $\dec(d,b)\in D$ we have $\pi\big(\lambda(d)\big) < \pi\big(\lambda(\dec(d,b))\big)$.
	\item \emph{reduced} if all decisions $d, d' \in D$ are \emph{essential} 
	(i.e.\ $\dec(d,0)\neq\dec(d,1)$) and $\pdd\ $ does not contain isomorphic 
	subdiagrams (i.e., $\pdd_d \cong \pdd_{d'}$ implies $d = d'$).
\end{enumerate}
To distinguish between components of a PDD $\pdd\ $ or BDD $\bdd$, we occasionally add suffixes, e.g.,
$\lambda_{\bdd}$ to indicate the labeling function of $\bdd$.
Note that our definition of PDDs covers the standard notions of DTs and BDDs:
A DT is a PDD where the underlying graph has a tree structure: 
for all $n,n'\in N$ with $n'\not\in\Reach(n)$ and $n\not\in\Reach(n')$ we have $\Reach(n)\cap\Reach(n')=\varnothing$.
A BDD is a reduced ordered PDD where all predicates are simple.

\paragraphh{Semantics of PDDs.}
Given a node $n$ in $\pdd$, we define a policy $\Sem{\pdd_n}$ 
recursively through evaluations $\epsilon\colon\Var\ra\Dom$:
\begin{center}$
	\Sem{\pdd_{n}}(\epsilon) = 
	\begin{cases}
		\lambda(n) & \text{if } n \in N{\setminus}D\\
		\Sem{\pdd_{\dec(n,0)}}(\epsilon) & \text{if } n\in D \text{ and } \epsilon ~\not\models~\lambda(n)\\
		\Sem{\pdd_{\dec(n,1)}}(\epsilon) & \text{if } n\in D \text{ and } \epsilon\models\lambda(n)\\
	\end{cases}
$\end{center}
Two PDDs $\pdd$ and $\pdd'$ are \emph{semantically equivalent} 
iff $\Sem{\pdd}({\epsilon}) = \Sem{\pdd'}(\epsilon)$ for all evaluations $\epsilon\colon\Var\ra\Dom$.
A path $d_0,b_0,d_1,\ldots,d_k \in (D\times \Bool)^\ast\times N$ is \emph{consistent}
with an evaluation $\epsilon\in\Eval$ if for all $i\leqslant k$ we have $\epsilon\models\lambda(d_i)$ iff $b_i=1$.
We call a PDD \emph{consistent} if for all paths from
the root to an action node there is a consistent evaluation.
Note that different to BDDs, the essentiality of reduced PDDs does not propagate
on a semantic level, i.e., a decision node can have different successors but still may be irrelevant for the overall represented policy. 
For instance, if we would replace $\{a,b\}$ in \Cref{fig:exampleDT} by $\{b\}$, the decision $x>0$
becomes irrelevant for the policy even though it has two different successors ($x>2$ and $\{b\}$). 
The reason is in the possible semantic dependence of predicates, leading to the policy yielding 
$\{a\}$ if $x>2$ and $\{b\}$ otherwise, independent of the decision $x>0$. 
Consistency however ensures that after reaching a predicate node there are always evaluations 
for both cases, satisfaction and unsatisfaction of the predicate.

\subsection{From PDDs to BDDs and Back}
Taking on a Boolean perspective on predicates yields equivalence classes for evaluations
that satisfy the same predicates. Then, predicates can be seen as Boolean variables
evaluated over equivalence class containment. Therefore, we now investigate
connections between PDDs and BDDs~\cite{ADD,MTBDD} and 
that enable mature theory and tool support of BDDs also for PDDs. 

\subsubsection{BDD Encoding}
Consider a PDD $\pdd=(N,\Var,\Act,\lambda,\dec,r)$ over $\Var$ and $\Dom$.
Let $\PVar$ denote a set of Boolean \emph{predicate variables} 
for which there is a bijection $\gamma\colon \PVar \ra \lambda(D)$ we call \emph{predicate bijection}.
Further, define the \emph{$\gamma$-lifting} of an evaluation $\epsilon\colon\Var\ra\Dom$ over
$\Var$ as the Boolean evaluation $\gamma^\epsilon\colon \PVar \ra \Bool$ where for all $x\in\PVar$:\vspace{-.5em}
\[
\gamma^\epsilon(x) \en{=}
\begin{cases}
	1 & \text{if }\epsilon\models\gamma(x)\\
	0 & \text{otherwise}
\end{cases}
\]
We then say that a BDD $\bdd$ over $\PVar$ is \emph{equivalent modulo $\gamma$} to a PDD $\pdd$
iff $\Sem{\bdd}(\gamma^\epsilon) = \Sem{\pdd\ }(\epsilon)$ for all $\epsilon\colon\Var\ra\Dom$. 
Any BDD $\bdd$ over $\PVar$ and $\gamma$ trivially
exhibits an equivalent PDD $\gamma(\bdd)$ that arises from $\bdd$ by changing the labeling
of each decision node $d \in D_{\bdd}$ to $\gamma(\lambda_{\bdd}(d))$.
A path in $\bdd$ is \emph{predicate consistent} w.r.t. $\gamma$ if its corresponding path 
in the PDD $\gamma(\bdd)$ is consistent.
Likewise, $\bdd$ is \emph{predicate consistent} w.r.t. $\gamma$ if $\gamma(\bdd)$ is consistent.
Note that we consider BDDs to be reduced and ordered, hence $\gamma(\bdd)$ is a reduced and ordered PDD.
The other way around, any reduced ordered PDD $\pdd\ $ directly provides a BDD by defining Boolean variables $\PVar$ 
for each predicate in $\pdd\ $ towards a predicate bijection $\gamma$ where each labeling of decision nodes
$d\in D$ are replaced by a predicate variable $\gamma^{-1}(\lambda(d))$.

The main advantages of BDDs, namely canonicity and concise representation,
cannot directly be expected for PDDs.
Our goal is hence to enable these desirable properties also for PDDs,
including possibilities to exploit mature theory and broad tool support available for BDDs. 
While ordering, reducedness, and (predicate) consistency are obviously maintained by the transformations
above, our main application concerns PDDs in form of DTs, which are a priori neither ordered nor reduced.

\begin{algorithm}[t]
	\small
	\begin{algorithmic}[1]
		\Require PDD $\pdd=(N,\Var,\Act,\lambda,\dec,r)$, bijection $\gamma\colon\PVar\ra\lambda(D)$, 
		$\PVar$ order $\pi$
		\Ensure A $\pi$-BDD $\bdd$
		\If{$\lambda(r)\subseteq\Act$} \Comment{action set labeled root}
		\State \Return 
		$\bdd=(\{r\},\PVar,\Act,\{(r,\lambda(r))\},\pi,\varnothing,r)$\label{l:base}
		\Else \Comment{predicate labeled root}
		\State $\bdd_0 \gets \textsc{Pdd2Bdd}(\pdd_{\dec(r,0)},\gamma,\pi)$ \Comment{compile 0-successor PDD}
		\State $\bdd_1 \gets \textsc{Pdd2Bdd}(\pdd_{\dec(r,1)},\gamma,\pi)$ \Comment{compile 1-successor PDD}
		\State \Return $\bdd=\textsc{ITE}(\gamma^{-1}(\lambda(r)), \bdd_1, \bdd_0)$ \label{l:ite}
		\EndIf 
	\end{algorithmic}
	\caption{\label{alg:compile} \textsc{Pdd2Bdd(\pdd,$\gamma$,$\pi$)}: compile PDD $\pdd$ to the $\pi$-BDD $\bdd$ w.r.t. $\gamma$}
\end{algorithm}

\Cref{alg:compile} compiles a PDD $\pdd\ $ into a BDD $\bdd$
that is equivalent to $\pdd$ modulo $\gamma$.
The algorithm recursively traverses the PDD until reaching an action node (see \Cref{l:base})
which is directly returned as BDD with an action node as root.
Here, we rely on the \textsc{ITE} operator (if-then-else)
to implement the interpretation of the root predicate $\lambda(r)$ by the BDD 
predicate variable $x=\gamma^{-1}(\lambda(r))$.
The \textsc{ITE} operator is standard for BDDs and implements the Shannon expansion of
$\Sem{\bdd}$, i.e., return $\Sem{\bdd_1}$ if $x$ holds and $\Sem{\bdd_0}$ otherwise.

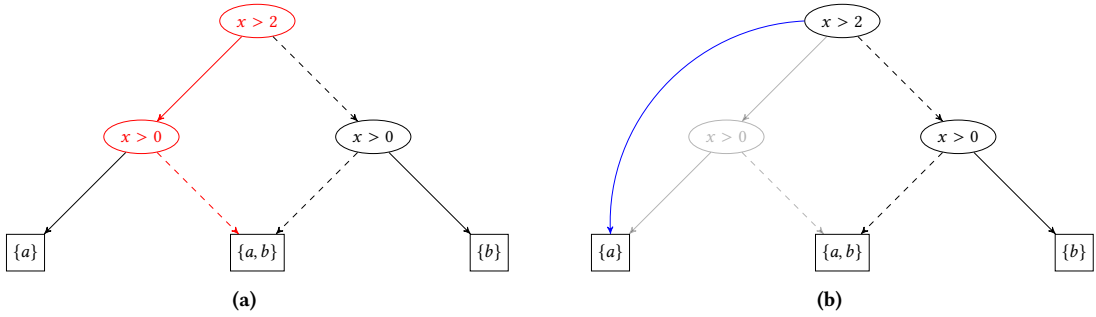
\begin{figure*}
	\centering
	\captionsetup[subfigure]{justification=centering}
	\subfloat[]{\label{fig:ex-inconsistant-pdd}
		\resizebox{!}{0.2\textwidth}{\tikzstyle{dn} = [draw, ellipse]
\tikzstyle{leaf} = [draw, minimum height=2em]
\tikzstyle{t} = []
\tikzstyle{f} = [dashed]
\begin{tikzpicture}
	\node[dn,red] at (4,2)   (p1)  {$x > 2$};
	\node[dn,red] at (2,0) (p01)     {$x > 0$};
	\node[dn] at (6,0) (p02)     {$x > 0$};
	
	\node[leaf] at (0, -2)  (a1)     {$\{a\}$};
	\node[leaf] at (4, -2) (a2)    {$\{a,b\}$};
	\node[leaf] at (8, -2) (a3)    {$\{b\}$};
%
	\draw[t,red] (p1) -- node[anchor=east, yshift=4pt]{}  (p01);
	\draw[f] (p1) -- node[anchor=west, yshift=4pt]{}  (p02);
	\draw[t] (p01) -- node[anchor=east, yshift=4pt]{}  (a1);
	\draw[f,red] (p01) -- node[anchor=west, yshift=4pt]{} (a2);
	\draw[f] (p02) -- node[anchor=east, yshift=4pt]{}  (a2);
	\draw[t] (p02) -- node[anchor=west, yshift=4pt]{} (a3);

\end{tikzpicture}}
	}\qquad
	\subfloat[]{\label{fig:ex-consistant-pdd}
		\resizebox{!}{0.2\textwidth}{\tikzstyle{dn} = [draw, ellipse]
\tikzstyle{leaf} = [draw, minimum height=2em]
\tikzstyle{t} = []
\tikzstyle{f} = [dashed]
\begin{tikzpicture}
	\node[dn] at (4,2)   (p1)  {$x > 2$};
	\node[dn, opacity=0.3] at (2,0) (p01)     {$x > 0$};
	\node[dn] at (6,0) (p02)     {$x > 0$};
	
	\node[leaf] at (0, -2)  (a1)     {$\{a\}$};
	\node[leaf] at (4, -2) (a2)    {$\{a,b\}$};
	\node[leaf] at (8, -2) (a3)    {$\{b\}$};
	%
	\draw[t, opacity=0.3] (p1) -- node[anchor=east, yshift=4pt]{}  (p01);
	\draw[f] (p1) -- node[anchor=west, yshift=4pt]{}  (p02);
	\draw[t, opacity=0.3] (p01) -- node[anchor=east, yshift=4pt]{}  (a1);
	\draw[f, opacity=0.3] (p01) -- node[anchor=west, yshift=4pt]{} (a2);
	\draw[f] (p02) -- node[anchor=east, yshift=4pt]{}  (a2);
	\draw[t] (p02) -- node[anchor=west, yshift=4pt]{} (a3);
		\draw[t, blue] (p1) edge[bend right=45] (a1);
	
\end{tikzpicture}}
	}%
	\caption{\label{fig:compile}Compilation of the DT in \Cref{fig:ex-tree} to a (predicate inconsistent) BDD (left) using \Cref{alg:compile}. A predicate consistent BDD (right) can be created by replacing the inconsistent part by the blue edge using \Cref{alg:pconsistency}.}
	\label{fig:examplePDD}
\end{figure*}

\begin{example}
\Cref{fig:ex-inconsistant-pdd} shows an application of \Cref{alg:compile} on the PDD in \Cref{fig:ex-tree} 
according to the variable
order $\pi=(p_1,p_0)$ over predicate variables $X=\{p_0,p_1\}$ with $\gamma(p_0)=(x>0)$
and $\gamma(p_1)=(x>2)$. Note that the resulting $\pi$-BDD is not predicate consistent,
since $x>2$ implies also $x>0$ and hence, $(x>2)\land\neg(x>0)$ is unsatisfiable,
leading the $0$-successor of $x>0$ to induce inconsistent paths (in red).
\end{example}

\begin{restatable}{lemma}{pddbdd}\label{lem:pdd2bdd}
	Given a PDD $\pdd$, a predicate bijection $\gamma$ for $\pdd$ over variables $\PVar$, and a total
	order $\pi$ over $\PVar$, \textsc{Pdd2Bdd}$(\pdd,\gamma,\pi)$ (see \Cref{alg:compile}) returns 
	a $\pi$-BDD that is equivalent modulo $\gamma$ to $\pdd$.
\end{restatable}

\begin{algorithm}[t]
\small
\begin{algorithmic}[1]
\Require $\pi$-BDD $\bdd=(N,\PVar,\Act,\lambda,\pi,\dec,r)$, $\gamma\colon\PVar\ra\Pred$, 
	context predicate formula $\phi$ over $\Pred$
\Ensure A $\pi$-BDD $\bdd'$ 
\If{$\lambda(r)\subseteq\Act$} \Comment{actions (base)}
	\State \Return $\bdd'=(\{r\},\PVar,\Act,\{(r,\lambda(r))\},\pi,\varnothing,r)$
\Else \Comment{predicates (recursion step)}
	\If{$\phi\land\gamma(\lambda(r))$ is satisfiable}\label{l:satpos}
		\State $\bdd_1' \gets \textsc{PConsistency}(\bdd_{\dec(r,1)},\gamma,\phi\land\gamma(\lambda(r)))$
		\If{$\phi\land\neg\gamma(\lambda(r))$ is satisfiable}\label{l:satneg}
			\State $\bdd_0' \gets \textsc{PConsistency}(\bdd_{\dec(r,0)},\gamma,\phi\land\neg\gamma(\lambda(r)))$
			\State \Return $\bdd' = \textsc{ITE}(\lambda(r), \bdd_1', \bdd_0')$ \label{l:bothsat}
		\Else
			\State \Return $\bdd'=\bdd_1'$ \label{l:s1sat}
		\EndIf
	\Else
		\State $\bdd_0' \gets \textsc{PConsistency}(\bdd_{\dec(r,0)},\gamma,\phi\land\neg\gamma(\lambda(r)))$
		\State \Return $\bdd' = \bdd_0'$
			\label{l:s0sat}
	\EndIf
\EndIf 
\end{algorithmic}
\caption{\label{alg:pconsistency} $\textsc{PConsistency}(\bdd,\gamma,\phi)$: 
Turn a $\pi$-BDD $\bdd$ into a $\pi$-BDD $\bdd'$ predicate consistent w.r.t. predicate bijection $\gamma$}
\end{algorithm}

\subsubsection{Predicate Consistent BDDs}
The BDD returned by \Cref{alg:compile} does not have to be predicate consistent.
Since PDDs are a priori not ordered, inconsistent orders of decisions might be enforced
by the order $\pi$ along \textsc{Pdd2Bdd} performs its compilation (see \Cref{fig:ex-inconsistant-pdd}).
We now present a simple yet effective method to prune inconsistent branches in
BDDs while maintaining their semantics in \Cref{alg:pconsistency}. Here, we rely on satisfiability solving 
over predicate formulas, an instance of a classical problem for state-of-the-art SMT solvers.
The algorithm traverses the BDD and keeps track of the visited predicates that were
considered to be satisfied or not. The latter is achieved by a \emph{context predicate formula} $\phi$
that is a conjunction over all past satisfaction decisions on predicates.
If one of the predicates assigned to the current decision node or its negation is unsatisfiable within
the context, then the opposite decision is taken and the corresponding decision node is returned.
Note that it cannot be that a predicate and its negation are both unsatisfiable.

\begin{restatable}{lemma}{pconsistency}\label{lem:pconsistency}
	For a given $\pi$-BDD $\bdd$ over $\PVar$ and a predicate bijection $\gamma\colon\PVar\ra\Pred$, 
	$\textsc{PConsistency}(\bdd,\gamma,\true)$ (see \Cref{alg:pconsistency})
	returns a predicate consistent $\pi$-BDD $\bdd'$ 
	that has the same semantics modulo $\gamma$ as $\bdd$, i.e., 
	$\Sem{\bdd}(\gamma^\epsilon)=\Sem{\bdd'}(\gamma^\epsilon)$ for all evaluations $\epsilon\colon\Var\ra\Dom$.
\end{restatable}

\begin{example}
\Cref{fig:ex-consistant-pdd} shows how to create a predicate-consistent $\pi$-BDD from the $\pi$-BDD 
in \Cref{fig:ex-consistant-pdd}. The left $p_0$-node (labeled by the predicate ``$x>0$'') can be safely removed and the $1$-successor of the
root is redirected to the leftmost $a$-node (i.e., by removing the greyed part and introducing
the blue branch).
\end{example}

Combining \Cref{lem:pdd2bdd} and \Cref{lem:pconsistency}, we obtain a method for predicate consistent 
BDD representations of a given PDD and can benefit from standard BDD techniques to 
represent control policies:

\begin{theorem}\label{thm:canonicityPDD}
	Given a PDD $\pdd$ and a total order $\pi$ on its predicates there is a 
	consistent reduced $\pi$-PDD represented by a predicate consistent BDD 
	semantically equivalent to $\pdd$.
\end{theorem}
\subsection{Towards PDD Explanations}
Due to BDDs admitting a canonical representation w.r.t. to a given
variable order and being reduced, BDDs admit minimal functional representations.
This renders BDDs also suitable for explanations: According to Occam’s razor, the best explanation 
for a phenomenon is the most simple one amongst all possible explanations (cf.~\cite{Goo77}).
Still, minimality heavily relies on fixing a variable order and a better explanation is
possible (i.e., a smaller BDD representation) for a different variable order.

\subsubsection{Variable-order Optimization}
Compared to DTs and hence general PDDs, enforcing a variable order comes at its price,
possibly countering concise representation. 
The latter phenomenon is well-known in the field of BDDs where switching to variable trees 
instead of orders can provide exponentially more compact representations~\cite{Dar11}.
Also different variable orders already might yield even more concise representations.

Thanks to our BDD representation of consistent reduced ordered PDDs (see \Cref{thm:canonicityPDD}), the advantages of 
BDDs carry over to PDDs, ensuring concise representations for a given variable order.
Fortunately, this holds also for well-known techniques applicable on BDDs
to mitigate variable order restrictions.
Deciding whether a given variable order is suboptimal is an NP-complete problem~\cite{Bollig}. 
To this end, heuristics have been developed to find good variable orders~(see, e.g., \cite{RicKul08,DubHusKaf24}).
One prominent method that is applicable on already constructed BDDs
is provided by Rudell's sifting method~\cite{Rud93}.
The idea is to permute adjacent variables in the variable order by
swapping levels of decision nodes in-place, leading to a permutation that provides smaller
diagram sizes. The result of one swap operation applied on the 
(consistent reduced ordered) PDD of \Cref{fig:exampleDT} on variables for predicates $x>0$ and $x>2$ 
can be seen in \Cref{fig:compile}. Here, the resulting BDD increases in size and thus,
such a swap might be not considered as favorable towards a better variable ordering.
However, sifting minimizes the BDDs taking on a Boolean interpretation of the variables.
If the BDD represents a PDD (see \Cref{lem:pdd2bdd}) the semantic interpretation of
predicates is not taken into account. To this end, sifting might introduce inconsistent branches, 
i.e., turn even a predicate consistent BDD into an inconsistent one. 
A similar observation has been already drawn when reordering linear DDs~\cite{LinearDD}.
In our previous example of the swap operation
onto \Cref{fig:exampleDT}, this phenomenon can be seen in \Cref{fig:compile}, where $\neg(x>0)$ cannot be satisfied after deciding for $x>2$.
Hence, BDD representations for PDDs can be further reduced after sifting using our
consistency transformation (\Cref{alg:pconsistency}).

\subsubsection{Care-set Reduction} 
In practice, controller policies are usually provided by partial functions, where decisions in certain 
states are not relevant and can be chosen arbitrarily. DT learning algorithms for their explanation exploit
these to reduce their size. Differently, PDDs and BDDs represent total functions,
not exploiting the full potential for even more concise representations.
Fortunately, this application is well-understood in the field of BDD-based symbolic verification, where 
fixed points over partial state domains are crucial towards performant verification algorithms. 
For this, Coudert and Madre introduced the \emph{restrict operator}~\cite{CouMad90} that 
minimizes a given BDD according to a \emph{care set} of variable evaluations.
Technically, nodes whose semantics agree in all evaluations in the care set are merged into existing nodes,
exploiting the sharing in BDDs.
In our setting, we apply the restrict operator on the domain of the controller function as care set.
We then obtain a reduced ordered consistent PDD that has a smaller size and yields the same outcomes for all original state-action pairs used to learn the (total control policy representing) DT.
The outcomes for other state evaluations than in the domain of the state-action pairs are used to reduce the PDD, leading to a total policy representation that is likely to differ from the one of the original DT.
Note that the restrict operation does not introduce new inconsistencies as it does not change the predicate
order in PDDs and only removes nodes on decision paths.

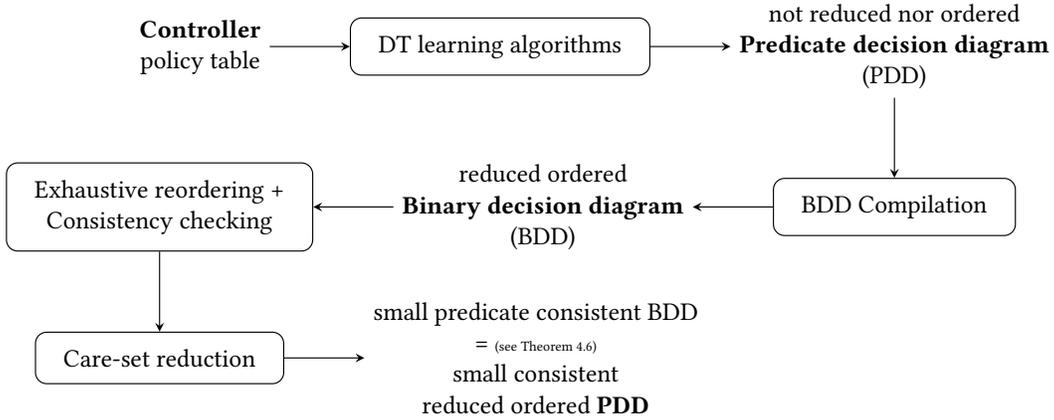
\begin{figure*}[t!]
	\centering
	\resizebox{0.8\textwidth}{!}{
\begin{tikzpicture}[node distance=1cm,>=stealth,align=center]
	\tikzstyle{algo} = [draw, rounded corners, rectangle, inner xsep=10pt, inner ysep=6pt]
	\tikzstyle{structure} = []
	
	\node (controller) [structure] {\textbf{Controller}\\policy table};
	\node (dtlearning) [algo, right=of controller] {DT learning algorithms};
	\node (pdd) [structure, right=of dtlearning] {not reduced nor ordered\\ \textbf{Predicate decision diagram}\\ (PDD)};
	
	\node (reordering) [algo, below=of pdd] {BDD Compilation};
	\node (consistencycheck) [structure, left=of reordering] {
		reduced ordered\\
		\textbf{Binary decision diagram}\\ (BDD)};
	\node (ei) [algo, left=of consistencycheck] {Exhaustive reordering + \\
														Consistency checking};
	\node (restricted) [algo, below=of ei] {Care-set reduction};
	\node (bdd) [structure, right=of restricted] {small predicate consistent BDD\\
	= {\tiny(see \cref{thm:canonicityPDD})}\\small consistent\\reduced ordered \textbf{PDD}};
	
	\draw[->] (controller) -- (dtlearning);
	\draw[->] (dtlearning) -- (pdd);
	\draw[->] (pdd) -- (reordering);
	\draw[->] (reordering) -- (consistencycheck);
	\draw[->] (consistencycheck) -- (ei);
	\draw[->] (ei) -- (restricted);
	\draw[->] (restricted) -- (bdd);

\end{tikzpicture}
	}
	\caption{PDD synthesis pipeline used in the experiments} 
	\label{fig:pipeline:eval}
\end{figure*}
\section{Experimental Evaluation}\label{sec:evaluation}

In this section, we describe our experimental setup and compare several kinds of controller representations:
BDDs with bit-blasting (bbBDDs), DTs, and reduced ordered PDDs (including reordered and consistent variants). 
Our experimental results show that, i) PDDs are almost as effective as a controller representation as DTs (answering \ref{rq1a}), ii) PDDs are more compact and explainable than bbBDDs (answering \ref{rq1b}), and iii) our reduction methods for PDDs are effective for controller representation (answering \ref{rq2}).

\inlineheadingbf{Synthesis Pipeline}
We implemented our approach in the tool \dtcontrol~\cite{dtcontrol2}
using a python wrapper to the BDD library \buddy~\cite{buddy}.
\Cref{fig:pipeline:eval} shows the pipeline we used to construct consistent reduced ordered PDDs
from tabular policies.
DTs are learnt from given controller policies (with axis-aligned predicates using entropy as the impurity measure), 
compiling them into BDDs by \Cref{alg:compile}, optimizing their variable ordering through reordering, and ensuring their predicate consistency before applying care-set reduction.

\inlineheadingbf{BDD Implementations}
We report on two different kinds of BDD representations for policies: bit-blasted BDDs (bbBDDs) and BDD representations for (reduced ordered) PDDs. 
The former follow the classical approach of a binary encoding of state variables' domains and constructing a BDD representation for the (partial) policy $\sigma\colon \Eval \ra \power{\Act}$ from the data set. The variable order for bbBDDs is initialized by replacing the order of state variables by the encoded bit vector blocks, followed by variable reordering through sifting until convergence.

For the BDDs constructed via our PDD pipeline, we initially chose a predicate variable order arising from a breadth-first traversal of the input DT. This ensured that the variable ordering used for BDD construction closely mirrors the structure of the learned DT.

Note that towards a more fair comparison we do not include the multiplicities of action nodes in DTs and only count decision nodes.
Further, to ensure a more fair comparison towards explainability, bbBDDs 
are not encoding actions binary and not choosing a complemented edge representation as done, e.g., in \cite{dtcontrol2}.

\inlineheadingbf{Benchmarks}
For evaluation, and to align with the existing comparison of DTs and bbBDDs, we use the standard benchmarks of the tool \dtcontrol~\cite{dtcontrol2}.\cwcomment{Maybe let's sell this as a feature... "To enhance the reproducibility and comparability of our evaluation, we utilize the identical set of benchmarks as outlined in..."} The benchmarks contain permissive policies for cyber-physical systems exported from  \scots~\cite{scots} and \uppaal~\cite{uppaal} and from the Quantitative Verification Benchmark Set~\cite{qcomp}, also including models from the \prism\ Benchmark Suite~\cite{KNP12b}. 
These case studies were solved using \storm~\cite{storm} and exported as \textsc{JSON} files. 
The policies given by \scots\ or \uppaal\ are permissive. \storm\ only exports deterministic policies.

\inlineheadingbf{Experiment Setup}
Our experiments were executed on a desktop machine with the following configuration:
one \SI{4.7}{GHz} CPU (AMD Ryzen\textsuperscript{\texttrademark} 7 PRO 5750G) with 16 processing units (virtual cores),
\SI{66}{GB} RAM, and 22.04.1-Ubuntu as operating system.

\inlineheadingbf{Extensions}
Our implementation also supports PDDs with linear predicates, i.e.,
where the decision node could contain an inequality of a linear combination of variables.
We restricted ourselves to axis-aligned predicates (that use only one variable in the decision nodes) as they are easier to interpret. Further,
in our initial experiments, we found that using linear predicates 
(as in the case of linear DDs~\cite{LinearDD})
also resulted in less sharing and thus larger PDDs.
For our pipeline, the reductions can be also put into different order. However, we here chose the order that led to the best results on average.

\subsection{Concise PDD Representation}

\begin{table*}[h!]
	\centering
	\caption{Sizes of controller representations: explicit states, as bit-blasted BDD, learnt DT, and in each step of the PDD pipeline. The upper part lists permissive policies from \scots\ and \uppaal, while the lower part lists deterministic policies from \storm. The smallest value in each row is written in bold. PDD explainability is evaluated w.r.t. final PDDs after care-set reduction.\label{tab:merged}}
	\setlength{\tabcolsep}{7pt}
	\begin{adjustbox}{width=\textwidth,center}
	\begin{tabular}{l|rrr|rrrr|r|rrr}
		\toprule
		 & \multicolumn{3}{c|}{Classical Representations} & \multicolumn{4}{c|}{PDD Pipeline} & \multicolumn{4}{c}{PDD Explainability}\\
		Controllers & States & bbBDD & DT & Plain & Reordered & Consistent & Care-set & {\#}shared & 
		$\frac{\mathit{DT}}{\mathit{bbBDD}}$ & $\frac{\mathit{PDD}}{\mathit{bbBDD}}$ & $\frac{\mathit{PDD}}{\mathit{DT}}$ \\[2mm]
		\midrule
		10rooms                 & 26244    & 1102  & 8648         		& 1332   & 419    & \textbf{344}   & \textbf{344}   & 211   & 7.85  & 0.31  & 0.04  \\
		cartpole                & 271      & 197   & \textbf{126} 		& 206    & 172    & 133   & \textbf{126}   & 0     & 0.64  & 0.64  & 1     \\
		cruise-latest           & 295615   & 2115  & 493          		& 1091   & 554    & 479   & \textbf{476}   & 66    & 0.23  & 0.23  & 0.97  \\
		dcdc                    & 593089   & 814   & \textbf{135} 		& 149    & \textbf{135}    & \textbf{135}   & \textbf{135}   & 0     & 0.17  & 0.17  & 1     \\
		helicopter              & 280539   & 3348  & 3169         		& 10294  & 5577   & 3276  & \textbf{3158}  & 486   & 0.95  & 0.94  & 1     \\
		traffic\_30m            & 16639662 & 4522  & 6286         		& 11088  & 5461   & 2497  & \textbf{2350}  & 1538  & 1.39  & 0.52  & 0.37  \\
		\midrule    
		beb.3-4.LineSeized      & 4173     & 1051  & \textbf{32}  		& 33     & \textbf{32}     & \textbf{32}    & \textbf{32}    & 0     & 0.03  & 0.03  & 1     \\
		blocksworld.5           & 1124     & 4043  & \textbf{617} 		& 2646   & 1742   & 1526  & 796            & 13    & 0.15  & 0.20  & 1.29  \\
		cdrive.10               & 1921     & 6151  & \textbf{1200}		& 9442   & 3828   & 3828  & \textbf{1200}  & 0     & 0.20  & 0.20  & 1     \\
		consensus.2.disagree    & 2064     & 112   & 33           		& 48     & 36     & 33    & \textbf{31}    & 1     & 0.29  & 0.28  & 0.94  \\
		csma.2-4.some\_before   & 7472     & 1172  & \textbf{51}  		& 78     & 60     & 58    & 53             & 2     & 0.04  & 0.05  & 1.04  \\
		eajs.2.100.5.ExpUtil    & 12627    & 1349  & 83  		& 108    & \textbf{81}     & \textbf{81}    & 85             & 7     & 0.06  & 0.06  & 1.02  \\
		echoring.MaxOffline1    & 104892   & 48765 & \textbf{934} 		& 4429   & 1288   & 1274  & 970            & 94    & 0.02  & 0.02  & 1.04  \\
		elevators.a-11-9        & 14742    & 6790  & 8163         		& 16563  & 6126   & 6077  & \textbf{5555}  & 2756  & 1.20  & 0.82  & 0.68  \\
		exploding-blocksworld.5 & 76741    & 39436 & \textbf{2490} 		& 14857  & 6648   & 6504  & 3635           & 1346  & 0.06  & 0.09  & 1.46  \\
		firewire\_abst.3.rounds & 610      & 51    & \textbf{12}  		& \textbf{12}     & \textbf{12}     & \textbf{12}    & \textbf{12}    & 0     & 0.24  & 0.24  & 1     \\
		ij.10                   & 1013     & \textbf{415} & 645   		& 458    & 452    & 452   & 453            & 222   & 1.55  & 1.09  & 0.70  \\
		pacman.5                & 232      & 440   & \textbf{21}  		& 28     & 23     & 23    & \textbf{21}    & 0     & 0.05  & 0.05  & 1     \\
		philosophers-mdp.3      & 344      & 251   & \textbf{195} 		& 310    & 246    & 205   & 203            & 5     & 0.78  & 0.81  & 1.04  \\
		pnueli-zuck.5           & 303427   & \textbf{59217} & 85685 	& 1304402 & 496033& 73139 & 72192          & 26941 & 1.45  & 1.22  & 0.84   \\
		rabin.3                 & 704      & 301   & \textbf{55}  		& 80     & 62     & 59    & 58             & 0     & 0.18  & 0.19  & 1.05  \\
		rectangle-tireworld.11  & 241      & 259   & \textbf{240} 		& 281    & 274    & \textbf{240}   & \textbf{240}   & 0     & 0.93  & 0.93  & 1     \\
		triangle-tireworld.9    & 48       & 38    & \textbf{13} 	 	& 14     & 14     & \textbf{13}    & \textbf{13}    & 0     & 0.34  & 0.34  & 1     \\
		wlan\_dl.0.80.deadline  & 189641   & 6591  & \textbf{1684}		& 13763  & 3513   & 2004  & 1824           & 199   & 0.26  & 0.28  & 1.08  \\
		zeroconf.1000.4.true 	& 1068 	   & 520   & \textbf{41}  		& 60     & 55     & 56    & 44             & 1     & 0.08  & 0.08  & 1.07 \\
		\midrule 
		\multicolumn{9}{r|}{Geometric means} & 0.27 & 0.23 & 0.84\\
		\bottomrule    
	\end{tabular}
	\end{adjustbox}
\end{table*}
Towards answering \ref{rq1a} and \ref{rq1b}, we conducted several experiments
those results are shown in \Cref{tab:merged}.
Here, we investigated the representation gaps between bbBDDs, PDDs, and DTs.

\subsubsection{Gap Between Decision Diagrams and DTs}
Using the classical BDD-based representation of control policies through bit-blasted BDDs (bbBDDs), there is a well-known gap between bbBDDs and DT sizes~\cite{dtcontrol2}.
The key question is whether PDDs can be used instead of DTs without efficiency drop (increase in size).
For this, we first report on the degree of closing the gap defined as
\begin{center}$R_{\mathit{gap}} = 
\left(\mathit{bbBDD}_{\mathit{}} - \mathit{PDD}_{\mathit{}}\right)
/
\left(\mathit{bbBDD}_{\mathit{}} - \mathit{DT}_{\mathit{}}\right)$\end{center}
where $\mathit{bbBDD}$, $\mathit{DT}$, and $\mathit{PDD}$ denote the number of decision nodes in the smallest bbBDD that we found, the DT generated using \dtcontrol, and the reduced ordered concise PDDs generated using our synthesis pipeline, respectively. 
An $R_{\mathit{gap}}$ value closer to $1$ means a PDD improves almost to the level of DTs in term of size,
while a value closer to $0$ means a PDD remains as bad as bbBDDs.
The computation excludes the cases (three in total)
where the DTs are larger than bbBDDs and the PDDs are smaller
(\emph{10rooms}, \emph{traffic\_30m},  and \emph{elevators.a-11-9}), i.e., where PDDs are even smaller than the DTs.
We bridge the gap between BDD and DT sizes on average by 88\%.
Consequently, we claim that PDD can safely replace DT without serious risks of efficiency decrease.

\begin{tcolorbox}[boxsep=1pt,left=3pt,right=3pt,top=1pt,bottom=1pt]
Towards \ref{rq1b} and \ref{rq1a}, reduced ordered consistent PDDs can close the well-known gap between bbBDDs and DTs for concise control policy representation.
\end{tcolorbox}

\subsubsection{Comparison With Bit-blasted BDDs}
For representing control policies, reduced ordered consistent PDDs can be expected to be more explainable than bbBDDs:
First, BDDs need multiple Boolean variables to represent state variables, which makes the representation less comprehensible. For instance, if there is a simple predicate $\mathit{temp}\leq 19$, then already at least five decisions on five Boolean variables (required to numerically represent 19 through bit-blasting) have to be made.
A single predicate used in a PDD cannot only express the conditions on state variables,
but also often captures the relation between multiple state variables easily.
An example for this can be found in \cref{app:tireworld}.

Second, contributing to explainability by Occam's Razor, PDDs can be significantly more concise compared to bbBDDs, as
we show in \Cref{tab:merged}. Here, PDDs are on (geometric) average
$77\%$ smaller than bbBDDs.
This is an improvement over DTs, which are $73\%$ smaller than bbBDDs.
In \Cref{fig:pdd_vs_bdd}, we provide an overview by comparing two normalized ratios:
(i) \emph{bbBDD size ratio} (ratio of the size of the constructed bbBDDs to the number of states in the system); and
(ii) \emph{PDD size ratio} (ratio of the size of the constructed PDDs to the number of states in the system).
With the exception of two cases (\emph{ij.10} and \emph{pnueli-zuck.5}), the PDDs are notably smaller than the bbBDDs.
In some cases (\emph{blocksworld.5} and \emph{cdrive.10}), the number of decision nodes in the bbBDDs are three times the size of the policy tables,
not offering a smaller representation of a policy.
In contrast, PDDs consistently succeed in providing compact representations.

\begin{figure}[t]
	\centering
		\includegraphics[width=0.7\linewidth]{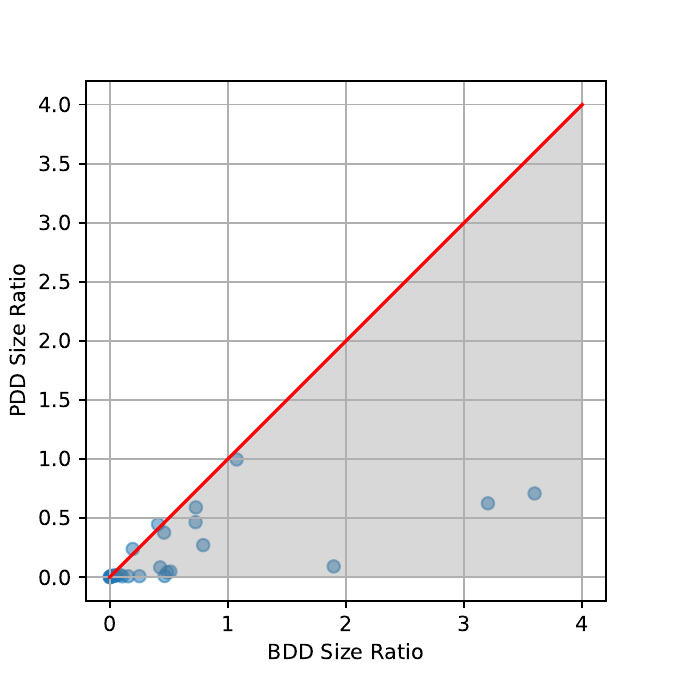}
		\caption{Comparison of normalized sizes of constructed PDDs and bbBDDs. Ratios of PDD (bbBDD, respectively) sizes to the number of states in the represented controller.}
		\label{fig:pdd_vs_bdd}
		\vspace{-1em}
\end{figure}%
\begin{tcolorbox}[boxsep=1pt,left=3pt,right=3pt,top=1pt,bottom=1pt]
For \ref{rq1b}, reduced ordered consistent PDDs provide on average a more concise representation of control policies than bbBDDs.
\end{tcolorbox}

\subsubsection{Comparison with DTs}
PDDs use predicates as in DTs, rendering decisions interpretable. 
However, due to merging isomorphic subdiagrams, common decision making can also be revealed through very same decision nodes, adding a component of explainability in contrast to DTs.
A similar improvement in explainability can also be observed in software engineering when avoiding code duplication~\cite{zuse2019software,DubKloPas24}.
Further, node sharing often decreases the size of the diagram, adding explainability following Occam's Razor.
Thus, to assess the effectiveness of PDDs for explainability, we need to consider the effect of sharing of nodes due the subgraph merging. 

In terms of size, we examine whether reduction rules (such as subgraph merging) compensate for the strict order of predicates.
As reported in \Cref{tab:merged}, PDDs are on (geometric) average
$16\%$ smaller than the DTs.
In \Cref{fig:pdd_vs_dt}, we provide an overview by comparing two normalized ratios:
(i) \emph{DT size ratio} (ratio of the size of the constructed DTs to the number of states in the system); and
(ii) \emph{PDD size ratio} (ratio of the size of the constructed PDDs to the number of states in the system).
In $8$ of the examples, the generated PDDs have the same size as the DTs. 
For some large controllers, we have produced smaller PDDs in comparison to the DT representation. 

\begin{figure}[t]
	\centering
		\includegraphics[width=0.7\linewidth]{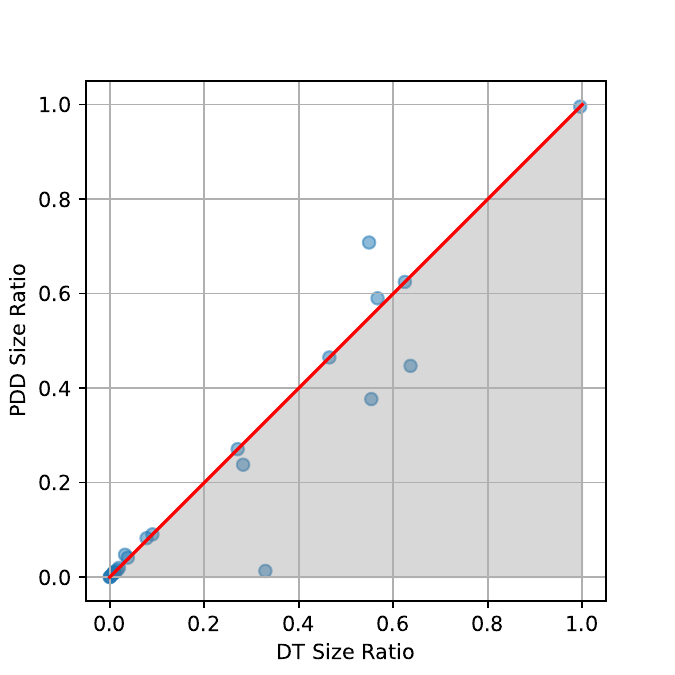}
		\caption{Comparison of normalized sizes of PDDs and DTs. For a structure $\mathcal{S} \in \{DT, PDD\}$, the $\mathcal{S}$ size ratio is the ratio of the size of $\mathcal{S}$ to the number of states $|S|$ in the system.}
		\label{fig:pdd_vs_dt}
\end{figure}
Is the case of \emph{10rooms}, where we compressed the DT by a factor of $25$ while producing the PDD.
Interestingly, the cases where bbBDDs provide smaller representation than DTs 
show that PDDs have either created the smallest representation among the three (in case of \emph{10rooms}, \emph{traffic\_30m}, \emph{elevators.a-11-9}) or managed to reduce the gap of size (in case of \emph{ij.10}, \emph{pnueli-zuck.5}).

The column marked with ``\#shared'' in \Cref{tab:merged} gives the number of \emph{shared nodes} in PDDs. In two cases (\emph{10rooms} and \emph{traffic\_30m}), the PDDs contain more than $60\%$ shared nodes (nodes with at least two parent edges).
Note that, the DT-learning algorithms are optimized to create small tree-like structures without node sharing.
As a result, only in larger models, we encounter isomorphic subdiagrams which can be merged while creating shared nodes in PDDs.

To illustrate this, we used the Israeli Jalfon randomized self-stabilizing protocol~\cite{ij} from the \prism\ benchmarks.
We extracted policies 
for different number of processes ($3$ to $17$)
and created PDD representations of the policies and observed the ratio of shared nodes in the PDD.
With increasing number of processes, the number of states in the system grows exponentially.
Smaller models have no shared nodes, but as the number of processes increases, the percentage of shared nodes in the PDD increases as well (See \Cref{fig:ij}; for a detailed description, see \cref{app:ij}).
For example, with $17$ processes, out of $34572$ nodes in the PDD, more than $80\%$ nodes have at least two parent edges. 
This demonstrates that PDDs become more efficient in larger models.
\begin{tcolorbox}[boxsep=1pt,left=3pt,right=3pt,top=1pt,bottom=1pt]
Answering \ref{rq1a}, PDDs provide on average a more concise representation of control policies than DTs,
largely benefitting from node sharing and common decision making.
\end{tcolorbox}

\subsection{Ablation Studies}
\Cref{tab:merged} shows the impact of each step of our PDD synthesis pipeline (see \Cref{fig:pipeline:eval}). 
Towards a fair comparison with DTs, we only counted the number of decision nodes and excluded duplicated action nodes that would massively add more nodes to DTs compared to DDs.
We see that imposing an order on predicates usually increases the number of nodes and does not outweigh the possibility of merging isomorphic subdiagrams. Most drastically, this can be seen with \emph{pnueli-zuck.5}, where the resulting PDD size is more than 15 times as big as the DT.
Reordering then has great impact, reducing the sizes of the PDDs significantly, especially for the larger examples. For \emph{pnueli-zuck.5}, it more than halves the number of nodes.
As apparent from \Cref{tab:merged}, consistency and care-set reduction have not always big impact but can also reduce PDD sizes: ensuring consistency can lead to 85\% reduction (\emph{pnueli-zuck.5}) and care-set reduction can reduce diagram sizes up to 68\% (\emph{cdrive.10}).

\begin{tcolorbox}[boxsep=1pt,left=3pt,right=3pt,top=1pt,bottom=1pt]
Towards \ref{rq2}, we conclude that consistency checking, reordering, and care-set reductions are all effective for reducing PDD sizes and thus improve control strategy explanation.
\end{tcolorbox}

\begin{figure}[t]
	\centering
	\includegraphics[width=\linewidth]{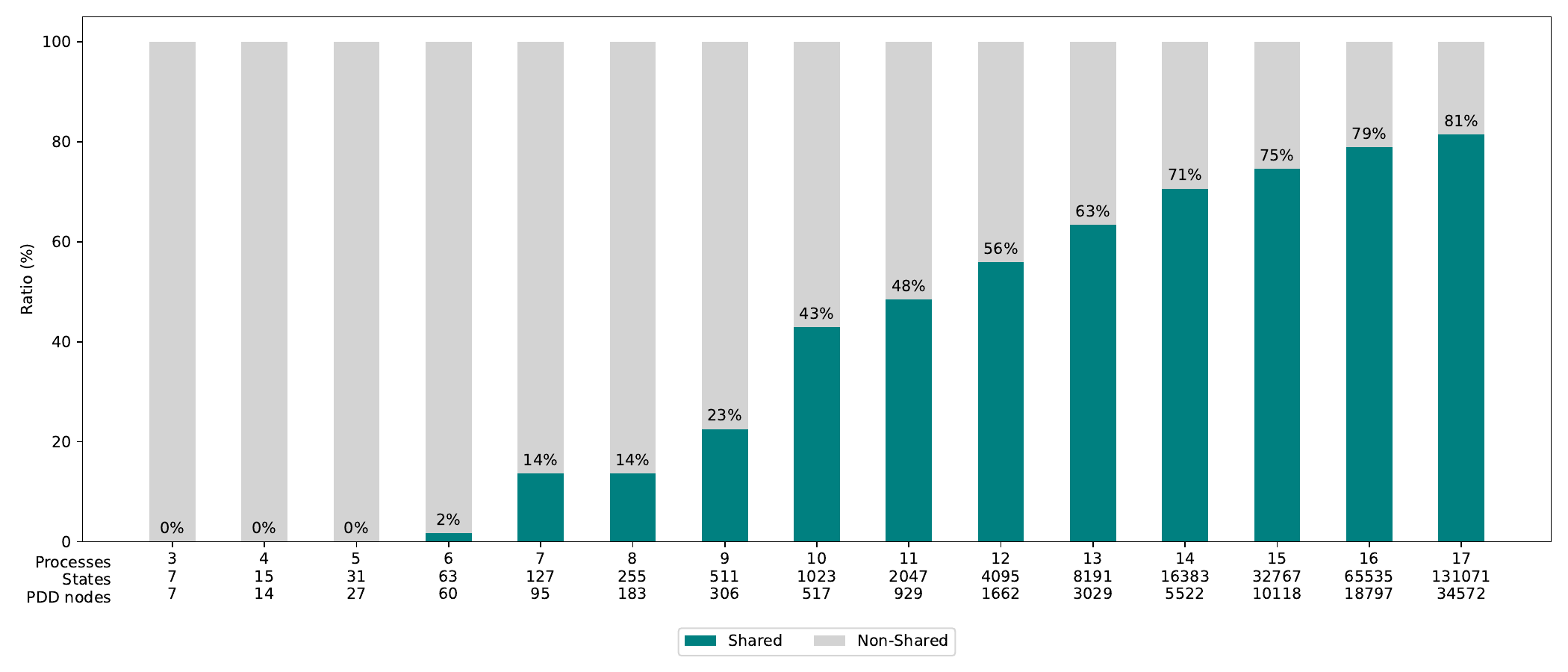}
	\caption{Fraction of shared nodes in PDDs for a self-stabilizing protocol with increasing number of processes.}
	\label{fig:ij}
\end{figure}
	
\section{Conclusion}

Binary decision diagrams (BDDs) provide concise representations of data for which mature theory and broad tool support is available. Differently, decision trees (DTs) are well-known for their interpretability and enjoy many performant learning algorithms.
Due to their predicates, DTs render also more explainable than BDDs with classical bit-blasting non-binary variables into bit-vectors.
We introduced \emph{predicate decision diagrams (PDDs)} along with a synthesis pipeline to generate PDDs from DTs and exploit BDD reduction techniques towards concise representations.
With PDDs, we established a representation for compact control policies that unite the benefits of both data structures, BDDs and DTs.
We found that the sizes of PDDs are on par with DT in most of the control policies we investigated.
This improves the state of the art in BDD-based representations such that PDDs can now be used as a viable alternative to DT as explainable data structure, but with also indicating common decision making through subdiagram merging.

For future work, an integration into state-of-the-art BDD packages~\cite{HusDubHer24} to directly support predicates and extend linear DDs~\cite{LinearDD} would improve applicability. Further, adaptions of approximation algorithms on PDDs \cite{RavMcMShi98,DubWir25} could extract the essence of control strategies in even smaller and thus more explainable PDDs.
It would be also interesting to establish theoretical guarantees on PDD sizes~\cite{SolPol23} or regarding different explainability metrics to underpin the benefits of PDDs.
We further envision a PDD learning algorithm that specifically prefers choices of predicates to increase sharing and could directly replace DT learning algorithms and benefit from the advantages of PDDs we showed in this paper.

\begin{acks}
	Authors in alphabetic order. This work was partially supported by the DFG under the projects
TRR 248 (see \url{https://perspicuous-computing.science}, project ID 389792660) and
EXC 2050/1 (CeTI, project ID 390696704, as part of Germany's Excellence Strategy), by the NWO through Veni grant VI.Veni.222.431, and the MUNI Award in Science and Humanities (MUNI/I/1757/2021) of the Grant Agency of Masaryk University.
\end{acks}

\bibliographystyle{ACM-Reference-Format}
\bibliography{ref}


\begin{thebibliography}{70}


\ifx \showCODEN    \undefined \def \showCODEN     #1{\unskip}     \fi
\ifx \showDOI      \undefined \def \showDOI       #1{#1}\fi
\ifx \showISBNx    \undefined \def \showISBNx     #1{\unskip}     \fi
\ifx \showISBNxiii \undefined \def \showISBNxiii  #1{\unskip}     \fi
\ifx \showISSN     \undefined \def \showISSN      #1{\unskip}     \fi
\ifx \showLCCN     \undefined \def \showLCCN      #1{\unskip}     \fi
\ifx \shownote     \undefined \def \shownote      #1{#1}          \fi
\ifx \showarticletitle \undefined \def \showarticletitle #1{#1}   \fi
\ifx \showURL      \undefined \def \showURL       {\relax}        \fi
\providecommand\bibfield[2]{#2}
\providecommand\bibinfo[2]{#2}
\providecommand\natexlab[1]{#1}
\providecommand\showeprint[2][]{arXiv:#2}

\bibitem[EUX(2019)]%
        {EUX19}
 \bibinfo{year}{2019}\natexlab{}.
\newblock \bibinfo{title}{Ethics Guidelines for Trustworthy {AI} - European
  Commission, Directorate-General for Communications Networks, Content and
  Technology}.
\newblock
\newblock
\newblock
\shownote{\url{https://data.europa.eu/doi/10.2759/177365}}.


\bibitem[USX(2020)]%
        {USX20}
 \bibinfo{year}{2020}\natexlab{}.
\newblock \bibinfo{title}{Four Principles of Explainable Artificial
  Intelligence - ({U.S.}) National Institute of Standards and Technology
  ({NIST})}.
\newblock
\newblock
\newblock
\shownote{\url{https://doi.org/10.6028/NIST.IR.8312-draft}}.


\bibitem[EUX(2021)]%
        {EUX21}
 \bibinfo{year}{2021}\natexlab{}.
\newblock \bibinfo{title}{Proposal for a Regulation laying down harmonised
  rules on artificial intelligence (Artificial Intelligence Act) and amending
  certain Union legislative acts, COM(2021) 206 final - European Commission}.
\newblock
\newblock
\newblock
\shownote{\url{https://ec.europa.eu/transparency/regdoc/rep/1/2021/EN/COM-2021-206-F1-EN-MAIN-PART-1.PDF}}.


\bibitem[Adadi and Berrada(2018)]%
        {adadi2018peeking}
\bibfield{author}{\bibinfo{person}{Amina Adadi} {and} \bibinfo{person}{Mohammed
  Berrada}.} \bibinfo{year}{2018}\natexlab{}.
\newblock \showarticletitle{Peeking inside the black-box: a survey on
  explainable artificial intelligence (XAI)}.
\newblock \bibinfo{journal}{\emph{IEEE access}}  \bibinfo{volume}{6}
  (\bibinfo{year}{2018}), \bibinfo{pages}{52138--52160}.
\newblock


\bibitem[Ashok et~al\mbox{.}(2019a)]%
        {viktor-qest19}
\bibfield{author}{\bibinfo{person}{Pranav Ashok}, \bibinfo{person}{Tom{\'{a}}s
  Br{\'{a}}zdil}, \bibinfo{person}{Krishnendu Chatterjee}, \bibinfo{person}{Jan
  K\v{r}et\'{i}nsk\'{y}}, \bibinfo{person}{Christoph~H. Lampert}, {and}
  \bibinfo{person}{Viktor Toman}.} \bibinfo{year}{2019}\natexlab{a}.
\newblock \showarticletitle{Strategy Representation by Decision Trees with
  Linear Classifiers}. In \bibinfo{booktitle}{\emph{{QEST}}}
  \emph{(\bibinfo{series}{Lecture Notes in Computer Science},
  Vol.~\bibinfo{volume}{11785})}. \bibinfo{publisher}{Springer},
  \bibinfo{pages}{109--128}.
\newblock


\bibitem[Ashok et~al\mbox{.}(2020)]%
        {dtcontrol}
\bibfield{author}{\bibinfo{person}{Pranav Ashok}, \bibinfo{person}{Mathias
  Jackermeier}, \bibinfo{person}{Pushpak Jagtap}, \bibinfo{person}{Jan
  Kret{\'{\i}}nsk{\'{y}}}, \bibinfo{person}{Maximilian Weininger}, {and}
  \bibinfo{person}{Majid Zamani}.} \bibinfo{year}{2020}\natexlab{}.
\newblock \showarticletitle{dtControl: decision tree learning algorithms for
  controller representation}. In \bibinfo{booktitle}{\emph{{HSCC} '20: 23rd
  {ACM} International Conference on Hybrid Systems: Computation and Control,
  Sydney, New South Wales, Australia, April 21-24, 2020}},
  \bibfield{editor}{\bibinfo{person}{Aaron~D. Ames}, \bibinfo{person}{Sanjit~A.
  Seshia}, {and} \bibinfo{person}{Jyotirmoy Deshmukh}} (Eds.).
  \bibinfo{publisher}{{ACM}}, \bibinfo{pages}{30:1--30:2}.
\newblock
\urldef\tempurl%
\url{https://doi.org/10.1145/3365365.3383468}
\showDOI{\tempurl}


\bibitem[Ashok et~al\mbox{.}(2021)]%
        {dtcontrol2}
\bibfield{author}{\bibinfo{person}{Pranav Ashok}, \bibinfo{person}{Mathias
  Jackermeier}, \bibinfo{person}{Jan Kret{\'{\i}}nsk{\'{y}}},
  \bibinfo{person}{Christoph Weinhuber}, \bibinfo{person}{Maximilian
  Weininger}, {and} \bibinfo{person}{Mayank Yadav}.}
  \bibinfo{year}{2021}\natexlab{}.
\newblock \showarticletitle{dtControl 2.0: Explainable Strategy Representation
  via Decision Tree Learning Steered by Experts}. In
  \bibinfo{booktitle}{\emph{{TACAS} {(2)}}} \emph{(\bibinfo{series}{Lecture
  Notes in Computer Science}, Vol.~\bibinfo{volume}{12652})}.
  \bibinfo{publisher}{Springer}, \bibinfo{pages}{326--345}.
\newblock


\bibitem[Ashok et~al\mbox{.}(2019b)]%
        {sos-qest19}
\bibfield{author}{\bibinfo{person}{Pranav Ashok}, \bibinfo{person}{Jan
  K\v{r}et\'{i}nsk\'{y}}, \bibinfo{person}{Kim~Guldstrand Larsen},
  \bibinfo{person}{Adrien~Le Co{\"{e}}nt}, \bibinfo{person}{Jakob~Haahr
  Taankvist}, {and} \bibinfo{person}{Maximilian Weininger}.}
  \bibinfo{year}{2019}\natexlab{b}.
\newblock \showarticletitle{{SOS:} Safe, Optimal and Small Strategies for
  Hybrid {M}arkov Decision Processes}. In \bibinfo{booktitle}{\emph{{QEST}}}
  \emph{(\bibinfo{series}{Lecture Notes in Computer Science},
  Vol.~\bibinfo{volume}{11785})}. \bibinfo{publisher}{Springer},
  \bibinfo{pages}{147--164}.
\newblock


\bibitem[Bahar et~al\mbox{.}(1997)]%
        {ADD}
\bibfield{author}{\bibinfo{person}{R.I. Bahar}, \bibinfo{person}{E.A. Frohm},
  \bibinfo{person}{C.M. Gaona}, \bibinfo{person}{G.D. Hachtel},
  \bibinfo{person}{E. Macii}, \bibinfo{person}{A. Pardo}, {and}
  \bibinfo{person}{F. Somenzi}.} \bibinfo{year}{1997}\natexlab{}.
\newblock \showarticletitle{Algebraic {{Decision Diagrams}} and {{Their
  Applications}}}.
\newblock \bibinfo{journal}{\emph{Formal Methods in System Design}}
  \bibinfo{volume}{10}, \bibinfo{number}{2} (\bibinfo{year}{1997}),
  \bibinfo{pages}{171--206}.
\newblock
\showISSN{1572-8102}


\bibitem[Barrett et~al\mbox{.}(2021)]%
        {BarSebSes21}
\bibfield{author}{\bibinfo{person}{Clark Barrett}, \bibinfo{person}{Roberto
  Sebastiani}, \bibinfo{person}{Sanjit Seshia}, {and} \bibinfo{person}{Cesare
  Tinelli}.} \bibinfo{year}{2021}\natexlab{}.
\newblock \showarticletitle{Satisfiability Modulo Theories}.
\newblock In \bibinfo{booktitle}{\emph{Handbook of Satisfiability, Second
  Edition}}, \bibfield{editor}{\bibinfo{person}{Armin Biere},
  \bibinfo{person}{Marijn J.~H. Heule}, \bibinfo{person}{Hans van Maaren},
  {and} \bibinfo{person}{Toby Walsh}} (Eds.). \bibinfo{series}{Frontiers in
  Artificial Intelligence and Applications}, Vol.~\bibinfo{volume}{336}.
  \bibinfo{publisher}{IOS Press}, Chapter~33, \bibinfo{pages}{825--885}.
\newblock
\urldef\tempurl%
\url{http://theory.stanford.edu/~barrett/pubs/BSST21.pdf}
\showURL{%
\tempurl}


\bibitem[Bollig and Wegener(1996)]%
        {Bollig}
\bibfield{author}{\bibinfo{person}{B. Bollig} {and} \bibinfo{person}{I.
  Wegener}.} \bibinfo{year}{1996}\natexlab{}.
\newblock \showarticletitle{Improving the variable ordering of OBDDs is
  NP-complete}.
\newblock \bibinfo{journal}{\emph{IEEE Trans. Comput.}} \bibinfo{volume}{45},
  \bibinfo{number}{9} (\bibinfo{year}{1996}), \bibinfo{pages}{993--1002}.
\newblock
\urldef\tempurl%
\url{https://doi.org/10.1109/12.537122}
\showDOI{\tempurl}


\bibitem[Boutilier et~al\mbox{.}(1995)]%
        {DBLP:conf/ijcai/BoutilierDG95}
\bibfield{author}{\bibinfo{person}{Craig Boutilier}, \bibinfo{person}{Richard
  Dearden}, {and} \bibinfo{person}{Mois{\'{e}}s Goldszmidt}.}
  \bibinfo{year}{1995}\natexlab{}.
\newblock \showarticletitle{Exploiting Structure in Policy Construction}. In
  \bibinfo{booktitle}{\emph{{IJCAI}}}. \bibinfo{publisher}{Morgan Kaufmann},
  \bibinfo{pages}{1104--1113}.
\newblock


\bibitem[Br{\'{a}}zdil et~al\mbox{.}(2015)]%
        {counterexample-learning}
\bibfield{author}{\bibinfo{person}{Tom{\'{a}}s Br{\'{a}}zdil},
  \bibinfo{person}{Krishnendu Chatterjee}, \bibinfo{person}{Martin Chmelik},
  \bibinfo{person}{Andreas Fellner}, {and} \bibinfo{person}{Jan
  K\v{r}et\'{i}nsk\'{y}}.} \bibinfo{year}{2015}\natexlab{}.
\newblock \showarticletitle{Counterexample Explanation by Learning Small
  Strategies in {M}arkov Decision Processes}. In
  \bibinfo{booktitle}{\emph{{CAV} {(1)}}} \emph{(\bibinfo{series}{Lecture Notes
  in Computer Science}, Vol.~\bibinfo{volume}{9206})}.
  \bibinfo{publisher}{Springer}, \bibinfo{pages}{158--177}.
\newblock


\bibitem[Br{\'{a}}zdil et~al\mbox{.}(2018)]%
        {viktor-reactivesynth}
\bibfield{author}{\bibinfo{person}{Tom{\'{a}}s Br{\'{a}}zdil},
  \bibinfo{person}{Krishnendu Chatterjee}, \bibinfo{person}{Jan
  K\v{r}et\'{i}nsk\'{y}}, {and} \bibinfo{person}{Viktor Toman}.}
  \bibinfo{year}{2018}\natexlab{}.
\newblock \showarticletitle{Strategy Representation by Decision Trees in
  Reactive Synthesis}. In \bibinfo{booktitle}{\emph{{TACAS} {(1)}}}
  \emph{(\bibinfo{series}{Lecture Notes in Computer Science},
  Vol.~\bibinfo{volume}{10805})}. \bibinfo{publisher}{Springer},
  \bibinfo{pages}{385--407}.
\newblock


\bibitem[Breiman et~al\mbox{.}(1984)]%
        {cart}
\bibfield{author}{\bibinfo{person}{Leo Breiman}, \bibinfo{person}{J.~H.
  Friedman}, \bibinfo{person}{R.~A. Olshen}, {and} \bibinfo{person}{C.~J.
  Stone}.} \bibinfo{year}{1984}\natexlab{}.
\newblock \bibinfo{booktitle}{\emph{Classification and Regression Trees}}.
\newblock \bibinfo{publisher}{Wadsworth}.
\newblock
\showISBNx{0-534-98053-8}


\bibitem[Bryant(1986)]%
        {BDD}
\bibfield{author}{\bibinfo{person}{Randal~E. Bryant}.}
  \bibinfo{year}{1986}\natexlab{}.
\newblock \showarticletitle{Graph-Based Algorithms for Boolean Function
  Manipulation}.
\newblock \bibinfo{journal}{\emph{{IEEE} Trans. Computers}}
  \bibinfo{volume}{35}, \bibinfo{number}{8} (\bibinfo{year}{1986}),
  \bibinfo{pages}{677--691}.
\newblock


\bibitem[Cabodi et~al\mbox{.}(2021)]%
        {DATE-Cabodi21}
\bibfield{author}{\bibinfo{person}{Gianpiero Cabodi}, \bibinfo{person}{Paolo~E.
  Camurati}, \bibinfo{person}{Alexey Ignatiev}, \bibinfo{person}{Jo{\~{a}}o
  Marques{-}Silva}, \bibinfo{person}{Marco Palena}, {and}
  \bibinfo{person}{Paolo Pasini}.} \bibinfo{year}{2021}\natexlab{}.
\newblock \showarticletitle{Optimizing Binary Decision Diagrams for
  Interpretable Machine Learning Classification}. In
  \bibinfo{booktitle}{\emph{Design, Automation {\&} Test in Europe Conference
  {\&} Exhibition, {DATE} 2021, Grenoble, France, February 1-5, 2021}}.
  \bibinfo{publisher}{{IEEE}}, \bibinfo{pages}{1122--1125}.
\newblock
\urldef\tempurl%
\url{https://doi.org/10.23919/DATE51398.2021.9474083}
\showDOI{\tempurl}


\bibitem[Cavada et~al\mbox{.}(2007)]%
        {CavCimFra07}
\bibfield{author}{\bibinfo{person}{Roberto Cavada}, \bibinfo{person}{Alessandro
  Cimatti}, \bibinfo{person}{Anders Franzen}, \bibinfo{person}{Krishnamani
  Kalyanasundaram}, \bibinfo{person}{Marco Roveri}, {and} \bibinfo{person}{R.K.
  Shyamasundar}.} \bibinfo{year}{2007}\natexlab{}.
\newblock \showarticletitle{Computing Predicate Abstractions by Integrating
  BDDs and SMT Solvers}. In \bibinfo{booktitle}{\emph{Formal Methods in
  Computer Aided Design (FMCAD'07)}}. \bibinfo{pages}{69--76}.
\newblock
\urldef\tempurl%
\url{https://doi.org/10.1109/FAMCAD.2007.35}
\showDOI{\tempurl}


\bibitem[Chaki et~al\mbox{.}(2009)]%
        {LinearDD}
\bibfield{author}{\bibinfo{person}{Sagar Chaki}, \bibinfo{person}{Arie
  Gurfinkel}, {and} \bibinfo{person}{Ofer Strichman}.}
  \bibinfo{year}{2009}\natexlab{}.
\newblock \showarticletitle{Decision diagrams for linear arithmetic}. In
  \bibinfo{booktitle}{\emph{Proceedings of 9th International Conference on
  Formal Methods in Computer-Aided Design, {FMCAD} 2009, 15-18 November 2009,
  Austin, Texas, {USA}}}. \bibinfo{publisher}{{IEEE}}, \bibinfo{pages}{53--60}.
\newblock
\urldef\tempurl%
\url{https://doi.org/10.1109/FMCAD.2009.5351143}
\showDOI{\tempurl}


\bibitem[Cheng and Yap(2005)]%
        {ConstrainedDD}
\bibfield{author}{\bibinfo{person}{Kenil C.~K. Cheng} {and}
  \bibinfo{person}{Roland H.~C. Yap}.} \bibinfo{year}{2005}\natexlab{}.
\newblock \showarticletitle{Constrained Decision Diagrams}. In
  \bibinfo{booktitle}{\emph{Proceedings, The Twentieth National Conference on
  Artificial Intelligence and the Seventeenth Innovative Applications of
  Artificial Intelligence Conference, July 9-13, 2005, Pittsburgh,
  Pennsylvania, {USA}}}, \bibfield{editor}{\bibinfo{person}{Manuela~M. Veloso}
  {and} \bibinfo{person}{Subbarao Kambhampati}} (Eds.).
  \bibinfo{publisher}{{AAAI} Press / The {MIT} Press},
  \bibinfo{pages}{366--371}.
\newblock
\urldef\tempurl%
\url{http://www.aaai.org/Library/AAAI/2005/aaai05-058.php}
\showURL{%
\tempurl}


\bibitem[Cimatti et~al\mbox{.}(1998)]%
        {DBLP:conf/aaai/CimattiRT98}
\bibfield{author}{\bibinfo{person}{Alessandro Cimatti}, \bibinfo{person}{Marco
  Roveri}, {and} \bibinfo{person}{Paolo Traverso}.}
  \bibinfo{year}{1998}\natexlab{}.
\newblock \showarticletitle{Automatic OBDD-Based Generation of Universal Plans
  in Non-Deterministic Domains}. In \bibinfo{booktitle}{\emph{{AAAI/IAAI}}}.
  \bibinfo{publisher}{{AAAI} Press / The {MIT} Press},
  \bibinfo{pages}{875--881}.
\newblock


\bibitem[Coudert and Madre(1990)]%
        {CouMad90}
\bibfield{author}{\bibinfo{person}{O. Coudert} {and} \bibinfo{person}{J.C.
  Madre}.} \bibinfo{year}{1990}\natexlab{}.
\newblock \showarticletitle{A unified framework for the formal verification of
  sequential circuits}. In \bibinfo{booktitle}{\emph{1990 IEEE International
  Conference on Computer-Aided Design. Digest of Technical Papers}}.
  \bibinfo{pages}{126--129}.
\newblock
\urldef\tempurl%
\url{https://doi.org/10.1109/ICCAD.1990.129859}
\showDOI{\tempurl}


\bibitem[Darwiche(2011)]%
        {Dar11}
\bibfield{author}{\bibinfo{person}{Adnan Darwiche}.}
  \bibinfo{year}{2011}\natexlab{}.
\newblock \showarticletitle{SDD: a new canonical representation of
  propositional knowledge bases}. In \bibinfo{booktitle}{\emph{Proceedings of
  the Twenty-Second International Joint Conference on Artificial Intelligence -
  Volume Volume Two}} (Barcelona, Catalonia, Spain)
  \emph{(\bibinfo{series}{IJCAI'11})}. \bibinfo{publisher}{AAAI Press},
  \bibinfo{pages}{819--826}.
\newblock
\showISBNx{9781577355144}


\bibitem[David et~al\mbox{.}(2015)]%
        {uppaal}
\bibfield{author}{\bibinfo{person}{Alexandre David},
  \bibinfo{person}{Peter~Gj{\o}l Jensen}, \bibinfo{person}{Kim~Guldstrand
  Larsen}, \bibinfo{person}{Marius Mikucionis}, {and}
  \bibinfo{person}{Jakob~Haahr Taankvist}.} \bibinfo{year}{2015}\natexlab{}.
\newblock \showarticletitle{Uppaal Stratego}. In
  \bibinfo{booktitle}{\emph{{TACAS}}} \emph{(\bibinfo{series}{Lecture Notes in
  Computer Science}, Vol.~\bibinfo{volume}{9035})}.
  \bibinfo{publisher}{Springer}, \bibinfo{pages}{206--211}.
\newblock


\bibitem[Dehnert et~al\mbox{.}(2017)]%
        {storm}
\bibfield{author}{\bibinfo{person}{Christian Dehnert},
  \bibinfo{person}{Sebastian Junges}, \bibinfo{person}{Joost{-}Pieter Katoen},
  {and} \bibinfo{person}{Matthias Volk}.} \bibinfo{year}{2017}\natexlab{}.
\newblock \showarticletitle{A Storm is Coming: {A} Modern Probabilistic Model
  Checker}. In \bibinfo{booktitle}{\emph{{CAV} {(2)}}}
  \emph{(\bibinfo{series}{Lecture Notes in Computer Science},
  Vol.~\bibinfo{volume}{10427})}. \bibinfo{publisher}{Springer},
  \bibinfo{pages}{592--600}.
\newblock


\bibitem[Della~Penna et~al\mbox{.}(2009)]%
        {DellaPenna2009}
\bibfield{author}{\bibinfo{person}{Giuseppe Della~Penna},
  \bibinfo{person}{Benedetto Intrigila}, \bibinfo{person}{Nadia Lauri}, {and}
  \bibinfo{person}{Daniele Magazzeni}.} \bibinfo{year}{2009}\natexlab{}.
\newblock \showarticletitle{Fast and Compact Encoding of Numerical Controllers
  Using OBDDs}.
\newblock In \bibinfo{booktitle}{\emph{Informatics in Control, Automation and
  Robotics: Selcted Papers from the International Conference on Informatics in
  Control, Automation and Robotics 2008}},
  \bibfield{editor}{\bibinfo{person}{Juan~Andrade Cetto},
  \bibinfo{person}{Jean-Louis Ferrier}, {and} \bibinfo{person}{Joaquim Filipe}}
  (Eds.). \bibinfo{publisher}{Springer Berlin Heidelberg},
  \bibinfo{address}{Berlin, Heidelberg}, \bibinfo{pages}{75--87}.
\newblock
\showISBNx{978-3-642-00271-7}


\bibitem[Dubslaff et~al\mbox{.}(2024a)]%
        {DubHusKaf24}
\bibfield{author}{\bibinfo{person}{Clemens Dubslaff}, \bibinfo{person}{Nils
  Husung}, {and} \bibinfo{person}{Nikolai K\"{a}fer}.}
  \bibinfo{year}{2024}\natexlab{a}.
\newblock \showarticletitle{Configuring BDD Compilation Techniques for Feature
  Models}. In \bibinfo{booktitle}{\emph{Proceedings of the 28th ACM
  International Systems and Software Product Line Conference}} (Dommeldange,
  Luxembourg) \emph{(\bibinfo{series}{SPLC '24})}.
  \bibinfo{publisher}{Association for Computing Machinery},
  \bibinfo{address}{New York, NY, USA}, \bibinfo{pages}{209--216}.
\newblock
\showISBNx{9798400705939}
\urldef\tempurl%
\url{https://doi.org/10.1145/3646548.3676538}
\showDOI{\tempurl}


\bibitem[Dubslaff et~al\mbox{.}(2024b)]%
        {DubKloPas24}
\bibfield{author}{\bibinfo{person}{Clemens Dubslaff}, \bibinfo{person}{Verena
  Kl{\"o}s}, {and} \bibinfo{person}{Juliane P{\"a}{\ss}ler}.}
  \bibinfo{year}{2024}\natexlab{b}.
\newblock \showarticletitle{Template Decision Diagrams for Meta Control and
  Explainability}. In \bibinfo{booktitle}{\emph{Explainable Artificial
  Intelligence}}, \bibfield{editor}{\bibinfo{person}{Luca Longo},
  \bibinfo{person}{Sebastian Lapuschkin}, {and} \bibinfo{person}{Christin
  Seifert}} (Eds.). \bibinfo{publisher}{Springer Nature Switzerland},
  \bibinfo{address}{Cham}, \bibinfo{pages}{219--242}.
\newblock
\showISBNx{978-3-031-63797-1}


\bibitem[Dubslaff and Wirtz(2025)]%
        {DubWir25}
\bibfield{author}{\bibinfo{person}{Clemens Dubslaff} {and}
  \bibinfo{person}{Joshua Wirtz}.} \bibinfo{year}{2025}\natexlab{}.
\newblock \bibinfo{booktitle}{\emph{Compiling Binary Decision Diagrams with
  Interrupt-Based Downsizing}}.
\newblock \bibinfo{publisher}{Springer Nature Switzerland},
  \bibinfo{address}{Cham}, \bibinfo{pages}{252--273}.
\newblock
\showISBNx{978-3-031-75778-5}
\urldef\tempurl%
\url{https://doi.org/10.1007/978-3-031-75778-5_12}
\showDOI{\tempurl}


\bibitem[E-Mart{\'\i}n et~al\mbox{.}(2014)]%
        {e2014progressive}
\bibfield{author}{\bibinfo{person}{Yolanda E-Mart{\'\i}n},
  \bibinfo{person}{Mar{\'\i}a~D R-Moreno}, {and} \bibinfo{person}{David~E
  Smith}.} \bibinfo{year}{2014}\natexlab{}.
\newblock \showarticletitle{Progressive heuristic search for probabilistic
  planning based on interaction estimates}.
\newblock \bibinfo{journal}{\emph{Expert Systems}} \bibinfo{volume}{31},
  \bibinfo{number}{5} (\bibinfo{year}{2014}), \bibinfo{pages}{421--436}.
\newblock


\bibitem[Florio et~al\mbox{.}(2023)]%
        {OptimalDDs_Classification}
\bibfield{author}{\bibinfo{person}{Alexandre~M. Florio}, \bibinfo{person}{Pedro
  Martins}, \bibinfo{person}{Maximilian Schiffer}, \bibinfo{person}{Thiago
  Serra}, {and} \bibinfo{person}{Thibaut Vidal}.}
  \bibinfo{year}{2023}\natexlab{}.
\newblock \showarticletitle{Optimal Decision Diagrams for Classification}.
\newblock \bibinfo{journal}{\emph{Proceedings of the AAAI Conference on
  Artificial Intelligence}} \bibinfo{volume}{37}, \bibinfo{number}{6}
  (\bibinfo{date}{Jun.} \bibinfo{year}{2023}), \bibinfo{pages}{7577--7585}.
\newblock
\urldef\tempurl%
\url{https://doi.org/10.1609/aaai.v37i6.25920}
\showDOI{\tempurl}


\bibitem[Fontaine and Gribomont(2002)]%
        {FonGri02}
\bibfield{author}{\bibinfo{person}{Pascal Fontaine} {and}
  \bibinfo{person}{E.~Pascal Gribomont}.} \bibinfo{year}{2002}\natexlab{}.
\newblock \showarticletitle{Using {BDDs} with Combinations of Theories}. In
  \bibinfo{booktitle}{\emph{Logic for Programming, Artificial Intelligence, and
  Reasoning}}, \bibfield{editor}{\bibinfo{person}{Matthias Baaz} {and}
  \bibinfo{person}{Andrei Voronkov}} (Eds.). \bibinfo{publisher}{Springer
  Berlin Heidelberg}, \bibinfo{address}{Berlin, Heidelberg},
  \bibinfo{pages}{190--201}.
\newblock
\showISBNx{978-3-540-36078-0}


\bibitem[Friso~Groote and van~de Pol(2000)]%
        {FriPol00}
\bibfield{author}{\bibinfo{person}{Jan Friso~Groote} {and}
  \bibinfo{person}{Jaco van~de Pol}.} \bibinfo{year}{2000}\natexlab{}.
\newblock \showarticletitle{Equational Binary Decision Diagrams}. In
  \bibinfo{booktitle}{\emph{Logic for Programming and Automated Reasoning}},
  \bibfield{editor}{\bibinfo{person}{Michel Parigot} {and}
  \bibinfo{person}{Andrei Voronkov}} (Eds.). \bibinfo{publisher}{Springer
  Berlin Heidelberg}, \bibinfo{address}{Berlin, Heidelberg},
  \bibinfo{pages}{161--178}.
\newblock
\showISBNx{978-3-540-44404-6}


\bibitem[Fujita et~al\mbox{.}(1997)]%
        {MTBDD}
\bibfield{author}{\bibinfo{person}{Masahiro Fujita},
  \bibinfo{person}{Patrick~C. McGeer}, {and} \bibinfo{person}{Jerry~Chih{-}Yuan
  Yang}.} \bibinfo{year}{1997}\natexlab{}.
\newblock \showarticletitle{Multi-Terminal Binary Decision Diagrams: An
  Efficient Data Structure for Matrix Representation}.
\newblock \bibinfo{journal}{\emph{Formal Methods Syst. Des.}}
  \bibinfo{volume}{10}, \bibinfo{number}{2/3} (\bibinfo{year}{1997}),
  \bibinfo{pages}{149--169}.
\newblock


\bibitem[Garg et~al\mbox{.}(2016)]%
        {neider-dt-invariants}
\bibfield{author}{\bibinfo{person}{Pranav Garg}, \bibinfo{person}{Daniel
  Neider}, \bibinfo{person}{P. Madhusudan}, {and} \bibinfo{person}{Dan Roth}.}
  \bibinfo{year}{2016}\natexlab{}.
\newblock \showarticletitle{Learning invariants using decision trees and
  implication counterexamples}. In \bibinfo{booktitle}{\emph{{POPL}}}.
  \bibinfo{publisher}{{ACM}}, \bibinfo{pages}{499--512}.
\newblock


\bibitem[Good(1977)]%
        {Goo77}
\bibfield{author}{\bibinfo{person}{I.~J. Good}.}
  \bibinfo{year}{1977}\natexlab{}.
\newblock \showarticletitle{Explicativity: a mathematical theory of explanation
  with statistical applications}.
\newblock \bibinfo{journal}{\emph{Proc. R. Soc. Lond.}}  \bibinfo{volume}{A354}
  (\bibinfo{year}{1977}), \bibinfo{pages}{303--330}.
\newblock


\bibitem[Gupta et~al\mbox{.}(2015)]%
        {gupta2015policy}
\bibfield{author}{\bibinfo{person}{Ujjwal~Das Gupta}, \bibinfo{person}{Erik
  Talvitie}, {and} \bibinfo{person}{Michael Bowling}.}
  \bibinfo{year}{2015}\natexlab{}.
\newblock \showarticletitle{Policy tree: Adaptive representation for policy
  gradient}. In \bibinfo{booktitle}{\emph{Proceedings of the AAAI Conference on
  Artificial Intelligence}}, Vol.~\bibinfo{volume}{29}.
\newblock


\bibitem[Hartmanns et~al\mbox{.}(2019)]%
        {qcomp}
\bibfield{author}{\bibinfo{person}{Arnd Hartmanns}, \bibinfo{person}{Michaela
  Klauck}, \bibinfo{person}{David Parker}, \bibinfo{person}{Tim Quatmann},
  {and} \bibinfo{person}{Enno Ruijters}.} \bibinfo{year}{2019}\natexlab{}.
\newblock \showarticletitle{The Quantitative Verification Benchmark Set}. In
  \bibinfo{booktitle}{\emph{{TACAS} {(1)}}} \emph{(\bibinfo{series}{Lecture
  Notes in Computer Science}, Vol.~\bibinfo{volume}{11427})}.
  \bibinfo{publisher}{Springer}, \bibinfo{pages}{344--350}.
\newblock


\bibitem[Hoey et~al\mbox{.}(1999)]%
        {DBLP:conf/uai/HoeySHB99}
\bibfield{author}{\bibinfo{person}{Jesse Hoey}, \bibinfo{person}{Robert
  St{-}Aubin}, \bibinfo{person}{Alan~J. Hu}, {and} \bibinfo{person}{Craig
  Boutilier}.} \bibinfo{year}{1999}\natexlab{}.
\newblock \showarticletitle{{SPUDD:} Stochastic Planning using Decision
  Diagrams}. In \bibinfo{booktitle}{\emph{{UAI}}}. \bibinfo{publisher}{Morgan
  Kaufmann}, \bibinfo{pages}{279--288}.
\newblock


\bibitem[Hu et~al\mbox{.}(2022)]%
        {AAAI-Hu22}
\bibfield{author}{\bibinfo{person}{Hao Hu}, \bibinfo{person}{Marie{-}Jos{\'{e}}
  Huguet}, {and} \bibinfo{person}{Mohamed Siala}.}
  \bibinfo{year}{2022}\natexlab{}.
\newblock \showarticletitle{Optimizing Binary Decision Diagrams with MaxSAT for
  Classification}. In \bibinfo{booktitle}{\emph{Thirty-Sixth {AAAI} Conference
  on Artificial Intelligence, {AAAI} 2022, Thirty-Fourth Conference on
  Innovative Applications of Artificial Intelligence, {IAAI} 2022, The Twelveth
  Symposium on Educational Advances in Artificial Intelligence, {EAAI} 2022
  Virtual Event, February 22 - March 1, 2022}}. \bibinfo{publisher}{{AAAI}
  Press}, \bibinfo{pages}{3767--3775}.
\newblock
\urldef\tempurl%
\url{https://doi.org/10.1609/AAAI.V36I4.20291}
\showDOI{\tempurl}


\bibitem[Husung et~al\mbox{.}(2024)]%
        {HusDubHer24}
\bibfield{author}{\bibinfo{person}{Nils Husung}, \bibinfo{person}{Clemens
  Dubslaff}, \bibinfo{person}{Holger Hermanns}, {and}
  \bibinfo{person}{Maximilian~A. K{\"o}hl}.} \bibinfo{year}{2024}\natexlab{}.
\newblock \showarticletitle{OxiDD}. In \bibinfo{booktitle}{\emph{Tools and
  Algorithms for the Construction and Analysis of Systems}},
  \bibfield{editor}{\bibinfo{person}{Bernd Finkbeiner} {and}
  \bibinfo{person}{Laura Kov{\'a}cs}} (Eds.). \bibinfo{publisher}{Springer
  Nature Switzerland}, \bibinfo{address}{Cham}, \bibinfo{pages}{255--275}.
\newblock
\showISBNx{978-3-031-57256-2}


\bibitem[Israeli and Jalfon(1990)]%
        {ij}
\bibfield{author}{\bibinfo{person}{Amos Israeli} {and} \bibinfo{person}{Marc
  Jalfon}.} \bibinfo{year}{1990}\natexlab{}.
\newblock \showarticletitle{Token management schemes and random walks yield
  self-stabilizing mutual exclusion}. In \bibinfo{booktitle}{\emph{Proceedings
  of the Ninth Annual ACM Symposium on Principles of Distributed Computing}}
  (Quebec City, Quebec, Canada) \emph{(\bibinfo{series}{PODC '90})}.
  \bibinfo{publisher}{Association for Computing Machinery},
  \bibinfo{address}{New York, NY, USA}, \bibinfo{pages}{119--131}.
\newblock
\showISBNx{089791404X}
\urldef\tempurl%
\url{https://doi.org/10.1145/93385.93409}
\showDOI{\tempurl}


\bibitem[Jr. et~al\mbox{.}(2010)]%
        {pessoa}
\bibfield{author}{\bibinfo{person}{Manuel~Mazo Jr.}, \bibinfo{person}{Anna
  Davitian}, {and} \bibinfo{person}{Paulo Tabuada}.}
  \bibinfo{year}{2010}\natexlab{}.
\newblock \showarticletitle{{PESSOA:} {A} Tool for Embedded Controller
  Synthesis}. In \bibinfo{booktitle}{\emph{{CAV}}}
  \emph{(\bibinfo{series}{Lecture Notes in Computer Science},
  Vol.~\bibinfo{volume}{6174})}. \bibinfo{publisher}{Springer},
  \bibinfo{pages}{566--569}.
\newblock


\bibitem[J{\"{u}}ngermann et~al\mbox{.}(2023)]%
        {DBLP:journals/sttt/JungermannKW23}
\bibfield{author}{\bibinfo{person}{Florian J{\"{u}}ngermann},
  \bibinfo{person}{Jan Kret{\'{\i}}nsk{\'{y}}}, {and}
  \bibinfo{person}{Maximilian Weininger}.} \bibinfo{year}{2023}\natexlab{}.
\newblock \showarticletitle{Algebraically explainable controllers: decision
  trees and support vector machines join forces}.
\newblock \bibinfo{journal}{\emph{Int. J. Softw. Tools Technol. Transf.}}
  \bibinfo{volume}{25}, \bibinfo{number}{3} (\bibinfo{year}{2023}),
  \bibinfo{pages}{249--266}.
\newblock


\bibitem[Kwiatkowska et~al\mbox{.}(2012)]%
        {KNP12b}
\bibfield{author}{\bibinfo{person}{M. Kwiatkowska}, \bibinfo{person}{G.
  Norman}, {and} \bibinfo{person}{D. Parker}.} \bibinfo{year}{2012}\natexlab{}.
\newblock \showarticletitle{The {PRISM} Benchmark Suite}. In
  \bibinfo{booktitle}{\emph{Proc. 9th International Conference on Quantitative
  Evaluation of SysTems (QEST'12)}}. \bibinfo{publisher}{IEEE CS Press},
  \bibinfo{pages}{203--204}.
\newblock


\bibitem[Lind-Nielsen(2004)]%
        {buddy}
\bibfield{author}{\bibinfo{person}{J{\o}rn Lind-Nielsen}.}
  \bibinfo{year}{2004}\natexlab{}.
\newblock \bibinfo{title}{{BuDDy}: A binary decision diagram package, version
  2.4}.
\newblock
\newblock
\urldef\tempurl%
\url{https://buddy.sourceforge.net/manual/}
\showURL{%
\tempurl}


\bibitem[McMillan(1993)]%
        {McM93}
\bibfield{author}{\bibinfo{person}{Kenneth~L. McMillan}.}
  \bibinfo{year}{1993}\natexlab{}.
\newblock \bibinfo{booktitle}{\emph{Symbolic Model Checking}}.
\newblock \bibinfo{publisher}{Springer US}, \bibinfo{address}{Boston, MA}.
\newblock


\bibitem[Miller(1993)]%
        {MDD}
\bibfield{author}{\bibinfo{person}{D.M. Miller}.}
  \bibinfo{year}{1993}\natexlab{}.
\newblock \showarticletitle{Multiple-valued logic design tools}. In
  \bibinfo{booktitle}{\emph{[1993] Proceedings of the Twenty-Third
  International Symposium on Multiple-Valued Logic}}. \bibinfo{pages}{2--11}.
\newblock
\urldef\tempurl%
\url{https://doi.org/10.1109/ISMVL.1993.289589}
\showDOI{\tempurl}


\bibitem[Minato(1996)]%
        {Min96}
\bibfield{author}{\bibinfo{person}{Shin-ichi Minato}.}
  \bibinfo{year}{1996}\natexlab{}.
\newblock \bibinfo{booktitle}{\emph{Binary decision diagrams and applications
  for VLSI CAD}}.
\newblock \bibinfo{publisher}{Kluwer Academic Publishers},
  \bibinfo{address}{USA}.
\newblock
\showISBNx{0792396529}


\bibitem[Mitchell(1997)]%
        {mitchellML}
\bibfield{author}{\bibinfo{person}{T.~M. Mitchell}.}
  \bibinfo{year}{1997}\natexlab{}.
\newblock \bibinfo{booktitle}{\emph{Machine learning}}.
\newblock \bibinfo{publisher}{McGraw-Hill}.
\newblock


\bibitem[Moctezuma et~al\mbox{.}(2012)]%
        {6408575}
\bibfield{author}{\bibinfo{person}{Luis E.~Gonzalez Moctezuma},
  \bibinfo{person}{Andrei Lobov}, {and} \bibinfo{person}{Jose L.~Martinez
  Lastra}.} \bibinfo{year}{2012}\natexlab{}.
\newblock \showarticletitle{Decision making by using tree-like structures on
  industrial controllers}. In \bibinfo{booktitle}{\emph{2012 Tenth
  International Conference on ICT and Knowledge Engineering}}.
  \bibinfo{pages}{77--83}.
\newblock
\urldef\tempurl%
\url{https://doi.org/10.1109/ICTKE.2012.6408575}
\showDOI{\tempurl}


\bibitem[M{\o}ller et~al\mbox{.}(1999)]%
        {DifferenceDD}
\bibfield{author}{\bibinfo{person}{Jesper~B. M{\o}ller}, \bibinfo{person}{Jakob
  Lichtenberg}, \bibinfo{person}{Henrik~Reif Andersen}, {and}
  \bibinfo{person}{Henrik Hulgaard}.} \bibinfo{year}{1999}\natexlab{}.
\newblock \showarticletitle{Difference Decision Diagrams}. In
  \bibinfo{booktitle}{\emph{Computer Science Logic, 13th International
  Workshop, {CSL} '99, 8th Annual Conference of the EACSL, Madrid, Spain,
  September 20-25, 1999, Proceedings}} \emph{(\bibinfo{series}{Lecture Notes in
  Computer Science}, Vol.~\bibinfo{volume}{1683})},
  \bibfield{editor}{\bibinfo{person}{J{\"{o}}rg Flum} {and}
  \bibinfo{person}{Mario Rodr{\'{\i}}guez{-}Artalejo}} (Eds.).
  \bibinfo{publisher}{Springer}, \bibinfo{pages}{111--125}.
\newblock
\urldef\tempurl%
\url{https://doi.org/10.1007/3-540-48168-0\_9}
\showDOI{\tempurl}


\bibitem[Nagayama et~al\mbox{.}(2007)]%
        {NDD}
\bibfield{author}{\bibinfo{person}{Shinobu Nagayama}, \bibinfo{person}{Tsutomu
  Sasao}, {and} \bibinfo{person}{Jon~T. Butler}.}
  \bibinfo{year}{2007}\natexlab{}.
\newblock \showarticletitle{Numerical Function Generators Using Edge-Valued
  Binary Decision Diagrams}. In \bibinfo{booktitle}{\emph{Proceedings of the
  12th Conference on Asia South Pacific Design Automation, {ASP-DAC} 2007,
  Yokohama, Japan, January 23-26, 2007}}. \bibinfo{publisher}{{IEEE} Computer
  Society}, \bibinfo{pages}{535--540}.
\newblock
\urldef\tempurl%
\url{https://doi.org/10.1109/ASPDAC.2007.358041}
\showDOI{\tempurl}


\bibitem[Neider and Markgraf(2019)]%
        {neider-fmcad19}
\bibfield{author}{\bibinfo{person}{Daniel Neider} {and} \bibinfo{person}{Oliver
  Markgraf}.} \bibinfo{year}{2019}\natexlab{}.
\newblock \showarticletitle{Learning-Based Synthesis of Safety Controllers}. In
  \bibinfo{booktitle}{\emph{{FMCAD}}}. \bibinfo{publisher}{{IEEE}},
  \bibinfo{pages}{120--128}.
\newblock


\bibitem[Pyeatt et~al\mbox{.}(2001)]%
        {PnH}
\bibfield{author}{\bibinfo{person}{Larry~D Pyeatt}, \bibinfo{person}{Adele~E
  Howe}, {et~al\mbox{.}}} \bibinfo{year}{2001}\natexlab{}.
\newblock \showarticletitle{Decision tree function approximation in
  reinforcement learning}. In \bibinfo{booktitle}{\emph{Proceedings of the
  third international symposium on adaptive systems: evolutionary computation
  and probabilistic graphical models}}, Vol.~\bibinfo{volume}{2}. Cuba,
  \bibinfo{pages}{70--77}.
\newblock


\bibitem[Quinlan(1986)]%
        {id3}
\bibfield{author}{\bibinfo{person}{J.~Ross Quinlan}.}
  \bibinfo{year}{1986}\natexlab{}.
\newblock \showarticletitle{Induction of Decision Trees}.
\newblock \bibinfo{journal}{\emph{Mach. Learn.}} \bibinfo{volume}{1},
  \bibinfo{number}{1} (\bibinfo{year}{1986}), \bibinfo{pages}{81--106}.
\newblock


\bibitem[Quinlan(1993)]%
        {c45}
\bibfield{author}{\bibinfo{person}{J.~Ross Quinlan}.}
  \bibinfo{year}{1993}\natexlab{}.
\newblock \bibinfo{booktitle}{\emph{{C4.5:} Programs for Machine Learning}}.
\newblock \bibinfo{publisher}{Morgan Kaufmann}.
\newblock
\showISBNx{1-55860-238-0}


\bibitem[Ravi et~al\mbox{.}(1998)]%
        {RavMcMShi98}
\bibfield{author}{\bibinfo{person}{Kavita Ravi}, \bibinfo{person}{Kenneth~L.
  McMillan}, \bibinfo{person}{Thomas~R. Shiple}, {and} \bibinfo{person}{Fabio
  Somenzi}.} \bibinfo{year}{1998}\natexlab{}.
\newblock \showarticletitle{Approximation and decomposition of binary decision
  diagrams}. In \bibinfo{booktitle}{\emph{Proceedings of the 35th annual Design
  Automation Conference}} (San Francisco, California, USA)
  \emph{(\bibinfo{series}{DAC '98})}. \bibinfo{publisher}{Association for
  Computing Machinery}, \bibinfo{address}{New York, NY, USA},
  \bibinfo{pages}{445--450}.
\newblock
\showISBNx{0897919645}
\urldef\tempurl%
\url{https://doi.org/10.1145/277044.277168}
\showDOI{\tempurl}


\bibitem[Rice and Kulhari(2008)]%
        {RicKul08}
\bibfield{author}{\bibinfo{person}{Michael Rice} {and} \bibinfo{person}{Sanjay
  Kulhari}.} \bibinfo{year}{2008}\natexlab{}.
\newblock \showarticletitle{A survey of static variable ordering heuristics for
  efficient {BDD}/{MDD} construction}.
\newblock \bibinfo{journal}{\emph{University of California, Tech. Rep}}
  (\bibinfo{year}{2008}), \bibinfo{pages}{130}.
\newblock


\bibitem[Rudell(1993)]%
        {Rud93}
\bibfield{author}{\bibinfo{person}{R. Rudell}.}
  \bibinfo{year}{1993}\natexlab{}.
\newblock \showarticletitle{Dynamic variable ordering for ordered binary
  decision diagrams}. In \bibinfo{booktitle}{\emph{Proc. of the IEEE/ACM
  Conference on Computer-Aided Design (ICCAD)}}. \bibinfo{publisher}{{IEEE}
  Computer Society}, \bibinfo{pages}{42--47}.
\newblock


\bibitem[Rungger and Zamani(2016)]%
        {scots}
\bibfield{author}{\bibinfo{person}{Matthias Rungger} {and}
  \bibinfo{person}{Majid Zamani}.} \bibinfo{year}{2016}\natexlab{}.
\newblock \showarticletitle{{SCOTS:} {A} Tool for the Synthesis of Symbolic
  Controllers}. In \bibinfo{booktitle}{\emph{{HSCC}}}.
  \bibinfo{publisher}{{ACM}}, \bibinfo{pages}{99--104}.
\newblock


\bibitem[S{\o}lvsten and van~de Pol(2023)]%
        {SolPol23}
\bibfield{author}{\bibinfo{person}{Steffan~Christ S{\o}lvsten} {and}
  \bibinfo{person}{Jaco van~de Pol}.} \bibinfo{year}{2023}\natexlab{}.
\newblock \showarticletitle{Predicting Memory Demands of BDD Operations Using
  Maximum Graph Cuts}. In \bibinfo{booktitle}{\emph{Automated Technology for
  Verification and Analysis}}, \bibfield{editor}{\bibinfo{person}{{\'E}tienne
  Andr{\'e}} {and} \bibinfo{person}{Jun Sun}} (Eds.).
  \bibinfo{publisher}{Springer Nature Switzerland}, \bibinfo{address}{Cham},
  \bibinfo{pages}{72--92}.
\newblock
\showISBNx{978-3-031-45332-8}


\bibitem[St{-}Aubin et~al\mbox{.}(2000)]%
        {DBLP:conf/nips/St-AubinHB00}
\bibfield{author}{\bibinfo{person}{Robert St{-}Aubin}, \bibinfo{person}{Jesse
  Hoey}, {and} \bibinfo{person}{Craig Boutilier}.}
  \bibinfo{year}{2000}\natexlab{}.
\newblock \showarticletitle{{APRICODD:} Approximate Policy Construction Using
  Decision Diagrams}. In \bibinfo{booktitle}{\emph{{NIPS}}}.
  \bibinfo{publisher}{{MIT} Press}, \bibinfo{pages}{1089--1095}.
\newblock


\bibitem[Tinelli and Harandi(1996)]%
        {TinHar96}
\bibfield{author}{\bibinfo{person}{Cesare Tinelli} {and} \bibinfo{person}{Mehdi
  Harandi}.} \bibinfo{year}{1996}\natexlab{}.
\newblock \showarticletitle{A New Correctness Proof of the Nelson-Oppen
  Combination Procedure}. In \bibinfo{booktitle}{\emph{Frontiers of Combining
  Systems: First International Workshop, Munich, March 1996}},
  \bibfield{editor}{\bibinfo{person}{Frans Baader} {and}
  \bibinfo{person}{Klaus~U. Schulz}} (Eds.). \bibinfo{publisher}{Springer
  Netherlands}, \bibinfo{address}{Dordrecht}, \bibinfo{pages}{103--119}.
\newblock
\showISBNx{978-94-009-0349-4}


\bibitem[Torfah et~al\mbox{.}(2021)]%
        {torfah2021synthesizing}
\bibfield{author}{\bibinfo{person}{Hazem Torfah}, \bibinfo{person}{Shetal
  Shah}, \bibinfo{person}{Supratik Chakraborty}, \bibinfo{person}{S Akshay},
  {and} \bibinfo{person}{Sanjit~A Seshia}.} \bibinfo{year}{2021}\natexlab{}.
\newblock \showarticletitle{Synthesizing pareto-optimal interpretations for
  black-box models}. In \bibinfo{booktitle}{\emph{Formal Methods in Computer
  Aided Design}}. \bibinfo{publisher}{IEEE}.
\newblock


\bibitem[Verhaeghe et~al\mbox{.}(2020)]%
        {verhaeghe2020learning}
\bibfield{author}{\bibinfo{person}{H{\'e}lene Verhaeghe},
  \bibinfo{person}{Siegfried Nijssen}, \bibinfo{person}{Gilles Pesant},
  \bibinfo{person}{Claude-Guy Quimper}, {and} \bibinfo{person}{Pierre Schaus}.}
  \bibinfo{year}{2020}\natexlab{}.
\newblock \showarticletitle{Learning optimal decision trees using constraint
  programming}.
\newblock \bibinfo{journal}{\emph{Constraints}}  \bibinfo{volume}{25}
  (\bibinfo{year}{2020}), \bibinfo{pages}{226--250}.
\newblock


\bibitem[Verwer and Zhang(2019)]%
        {verwer2019learning}
\bibfield{author}{\bibinfo{person}{Sicco Verwer} {and}
  \bibinfo{person}{Yingqian Zhang}.} \bibinfo{year}{2019}\natexlab{}.
\newblock \showarticletitle{Learning optimal classification trees using a
  binary linear program formulation}. In \bibinfo{booktitle}{\emph{Proceedings
  of the AAAI conference on artificial intelligence}},
  Vol.~\bibinfo{volume}{33}. \bibinfo{pages}{1625--1632}.
\newblock


\bibitem[Vos and Verwer(2023)]%
        {vos2023optimal}
\bibfield{author}{\bibinfo{person}{Dani{\"e}l Vos} {and} \bibinfo{person}{Sicco
  Verwer}.} \bibinfo{year}{2023}\natexlab{}.
\newblock \showarticletitle{Optimal decision tree policies for Markov decision
  processes}. In \bibinfo{booktitle}{\emph{Proceedings of the Thirty-Second
  International Joint Conference on Artificial Intelligence}}.
  \bibinfo{pages}{5457--5465}.
\newblock


\bibitem[Zapreev et~al\mbox{.}(2018)]%
        {DBLP:conf/adhs/ZapreevVM18}
\bibfield{author}{\bibinfo{person}{Ivan~S. Zapreev}, \bibinfo{person}{Cees
  Verdier}, {and} \bibinfo{person}{Manuel~Mazo Jr.}}
  \bibinfo{year}{2018}\natexlab{}.
\newblock \showarticletitle{Optimal Symbolic Controllers Determinization for
  {BDD} storage}. In \bibinfo{booktitle}{\emph{{ADHS} 2018}}
  \emph{(\bibinfo{series}{IFAC-PapersOnLine}, Vol.~\bibinfo{volume}{51})}.
  \bibinfo{publisher}{Elsevier}, \bibinfo{pages}{1--6}.
\newblock
\urldef\tempurl%
\url{https://doi.org/10.1016/j.ifacol.2018.08.001}
\showDOI{\tempurl}


\bibitem[Zuse(2019)]%
        {zuse2019software}
\bibfield{author}{\bibinfo{person}{Horst Zuse}.}
  \bibinfo{year}{2019}\natexlab{}.
\newblock \bibinfo{booktitle}{\emph{Software complexity: measures and
  methods}}. Vol.~\bibinfo{volume}{4}.
\newblock \bibinfo{publisher}{Walter de Gruyter GmbH \& Co KG}.
\newblock


\end{thebibliography}

\clearpage
\appendix
\section*{Appendix}
\renewcommand{\thesubsection}{\Alph{subsection}}
\sisetup{round-mode=places,round-precision=3}

\subsection{Proofs of \cref{sec:MTPDDs}}
\pddbdd*
\begin{proof}
	Let $\pdd=(N,\Var,\Act,\lambda,\dec,r)$, $\gamma\colon\PVar\ra\lambda(D)$, 
	$\pi$ a total order over $\PVar$, and $\epsilon\colon\Var\ra\Dom$.
	We prove the statement by induction.
	For the base case of the induction, if $\lambda(r)\subseteq\Act$, we clearly have $\Sem{\bdd}=\lambda(r)=\Sem{\pdd}$ and
	$\bdd$ is a $\pi$-BDD (see \cref{l:base}). 
	For the induction step, for $b\in\Bool$ assume $\Sem{\bdd_b}(\gamma^\epsilon)=\Sem{\pdd_{\dec(r,b)}}(\epsilon)$. 
	Then let $i=\gamma^\epsilon(\gamma^{-1}(\lambda(r))$ and thus,
	$\Sem{\bdd}(\gamma^\epsilon)=\Sem{\bdd_i}(\gamma^\epsilon)$ according to the semantics 
	of the \textsc{ITE} operator (see \cref{l:ite}). 
	Hence, $i=1$ iff $\epsilon\models\gamma(\gamma^{-1}(\lambda(r)))$ iff 
	$\epsilon\models\lambda(r)$ since $\gamma$ is a bijection. 
	By the definition of PDD semantics,
	$\Sem{\pdd}(\epsilon)=\Sem{\pdd_{\dec(r,i)}}(\epsilon)$ and hence, $\Sem{\bdd}(\gamma^\epsilon)=
	\Sem{\pdd}(\epsilon)$.
\end{proof}

\pconsistency*
\begin{proof}
	Let us first assume that $\epsilon\models\phi$ and 
	show that $\bdd'=\textsc{PConsistency}(\bdd,\gamma,\phi)$ preserves
	semantics, i.e., $\Sem{\bdd}(\gamma^\epsilon)=\Sem{\bdd'}(\gamma^\epsilon)$.
	We prove the statement by induction.
	For the base case of the induction, for $\lambda(r)\subseteq\Act$, we clearly have 
	$\Sem{\bdd}(\gamma^\epsilon)=\lambda(r)=\Sem{\bdd'}(\gamma^\epsilon)$ and
	$\bdd'$ is a $\pi$-BDD (see \cref{l:base}). 
	For the induction step, assume that $\Sem{\bdd'_1}(\gamma^\epsilon)=
	\Sem{\bdd_{\dec(r,1)}}(\gamma^\epsilon)$ and $\Sem{\bdd'_0}(\gamma^\epsilon)=
	\Sem{\bdd_{\dec(r,0)}}(\gamma^\epsilon)$.
	If $\epsilon\models\gamma(\lambda(r))$ then $\gamma^\epsilon(\lambda(r)) = 1$, which leads to 
	$\Sem{\bdd}(\gamma^\epsilon) = \Sem{\bdd_{\dec(r,1)}}(\gamma^\epsilon)$
	by the semantics of BDDs. Further, $\phi\land\gamma(\lambda(r))$ is clearly satisfiable 
	(witnessed by $\epsilon$) and thus $\Sem{\bdd'}_\phi(\gamma^\epsilon) = \Sem{\bdd'_1}(\gamma^\epsilon)$
	by the definition of the ITE operator (\cref{l:bothsat}) or directly by
	setting $\bdd'$ to $\bdd'_1$ (\cref{l:s1sat}). 
	Hence, $\Sem{\bdd}(\gamma^\epsilon)= \Sem{\bdd'}(\gamma^\epsilon)$ by induction hypothesis
	for the case $\epsilon\models\gamma(\lambda(r))$.
	Likewise, if $\epsilon\models\neg\gamma(\lambda(r))$ then $\gamma^\epsilon(\lambda(r)) = 0$, which leads to 
	$\Sem{\bdd}_\phi(\gamma^\epsilon) = \Sem{\bdd_{\dec(r,0)}}(\gamma^\epsilon)$ and
	$\phi\land\neg\gamma(\lambda(r))$ being satisfiable. Thus, 
	$\Sem{\bdd'}(\gamma^\epsilon) = \Sem{\bdd'_0}(\gamma^\epsilon)$
	by the definition of the ITE operator (\cref{l:bothsat}) or directly by
	setting $\bdd'$ to $\bdd'_0$ (\cref{l:s0sat}). By induction hypothesis, we obtain
	$\Sem{\bdd}(\gamma^\epsilon)= \Sem{\bdd'}(\gamma^\epsilon)$ 
	for the case $\epsilon\models\neg\gamma(\lambda(r))$.
	\smallskip
	
	\noindent It is left to show that $\bdd'$ is predicate consistent.
	Let $\rho$ be a path in $\bdd'$ where\vspace{-1em}
	\[ \rho = d_0,b_0,d_1,\ldots,d_k \in (D'\times \Bool)^\ast\times (N'{\setminus}D')
	\]
	Then for each $i<k$ the node $d_i$ is the result of an ITE operation (\cref{l:bothsat})
	in a call of $\textsc{PConsistency}$ with context $\phi_i$, since this is the only occasion where 
	nodes are generated during recursive calls of $\textsc{PConsistency}$. 
	Towards reaching \cref{l:bothsat}, $\phi_i\land\gamma(\lambda'(d_i))$ and 
	$\phi_i\land\neg\gamma(\lambda'(d_i))$ have been both be satisfiable where the conjunction $\phi_i$
	contains all $\gamma(\lambda'(d_j))$ for $j<i$ with $b_j=1$ and $\neg\gamma(\lambda'(d_j))$
	for $j<i$ for $b_j=0$. Thus, for all $i<k$ there are evaluations $\epsilon^0_i,\epsilon^1_i\colon\Var\ra\Dom$ 
	where $\epsilon^0_i\models\phi_i\land\neg\gamma(\lambda'(d_i))$ and 
	$\epsilon^1_i\models\phi_i\land\gamma(\lambda'(d_i))$, leading to $\epsilon=\epsilon^{b_{k-1}}_{k-1}$ being
	an evaluation where $\gamma^\epsilon(\lambda(d_i)) = b_i$ for all $i<k$.
	Hence, $\epsilon$ serves as witness for $\rho$ being predicate consistent.
\end{proof}

\subsection{Case Study: Triangle-tireworld}\label{app:tireworld}
Consider the example of triangle tireworld~\cite{e2014progressive} from the QCOMP benchmark set~\cite{qcomp} as describled in \cref{fig:tireworld}.
The objective is to find the policy to reach the location $2$ with maximum probability. 
This system is modeled as a {\textsc{JANI}} file. It contains multiple state variables:
$l\in\{0,5\}$ denoting the locations, \emph{has\_spare} denoting where the car already has a spare tire, \emph{spare\_$i$} for $i\in\{3,4,5\}$ denotes if there is a spare tire at location $i$.

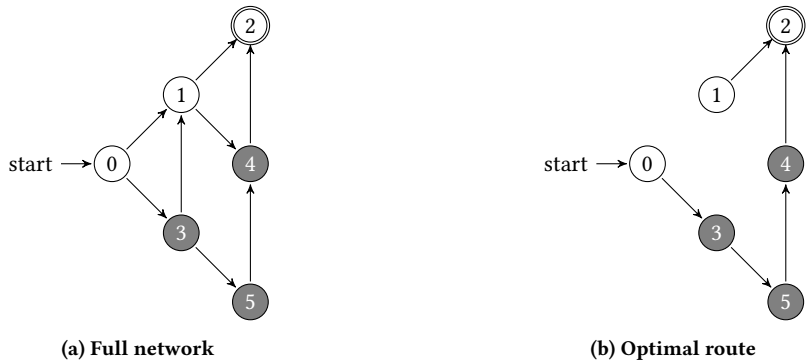
\begin{figure*}[h]
	\centering
	\begin{subfigure}{0.4\textwidth}
		\centering
		\begin{tikzpicture}[node distance=1.3cm,on grid,
	]
	\tikzset{
		mynode/.style={state, minimum size=1.5em, inner sep=1pt}
	}	
	\node[mynode, initial] 	(s0)						{$0$};
	\node[mynode]         	(s1) [above right=of s0] 	{$1$};
	\node[mynode, accepting]		(s2) [above right=of s1] 	{$2$};
	\node[mynode, fill=gray, text=white]         	(s3) [below right=of s0] 	{$3$};
	\node[mynode, fill=gray, text=white]         	(s4) [below right =of s1] 	{$4$};
	\node[mynode, fill=gray, text=white]         	(s5) [below right=of s3] 	{$5$};
\draw[->] (s0) -- (s1);
\draw[->] (s1) -- (s2);	
\draw[->] (s0) -- (s3);
\draw[->] (s3) -- (s5);	
\draw[->] (s1) -- (s4);
\draw[->] (s5) -- (s4);	
\draw[->] (s4) -- (s2);
\draw[->] (s3) -- (s1);

\end{tikzpicture}
		\caption{Full network}
		\label{fig:tireworld}
	\end{subfigure}%
	\begin{subfigure}{0.4\textwidth}
		\centering
		\begin{tikzpicture}[node distance=1.3cm,on grid,
	]
	\tikzset{
		mynode/.style={state, minimum size=1.5em, inner sep=1pt}
	}	
	\node[mynode, initial] 	(s0)						{$0$};
	\node[mynode]         	(s1) [above right=of s0] 	{$1$};
	\node[mynode, accepting]		(s2) [above right=of s1] 	{$2$};
	\node[mynode, fill=gray, text=white]         	(s3) [below right=of s0] 	{$3$};
	\node[mynode, fill=gray, text=white]         	(s4) [below right =of s1] 	{$4$};
	\node[mynode, fill=gray, text=white]         	(s5) [below right=of s3] 	{$5$};
	%
	\draw[->] (s1) -- (s2);	
	\draw[->] (s0) -- (s3);
	\draw[->] (s3) -- (s5);	
	\draw[->] (s5) -- (s4);	
	\draw[->] (s4) -- (s2);
	
\end{tikzpicture}
		\caption{Optimal route}
		\label{fig:tireworld-policy}
	\end{subfigure}
	\caption{Triangle tireworld : A car has to move from location $0$ to location $2$ in a directed graph like network (Fig. a). 
		When the car drives a segment of the route, with $0.5$ probability, the car can get a flat tire. 
		Spare tires are only available in locations $3$, $4$ and $5$. 
		The car starts with a spare tire.
		The optimal route (Fig. b) is suggested by the policy generated by \storm.
}
	\label{fig:tireworld-both}
\end{figure*}
Using \storm, we construct a controller policy defined for $48$ states. 
The PDD generated from the policy is described in \cref{fig:PDD-tireworld}.
In comparison to a policy table or an BDD, this representation is not only smaller, but also explainable:
\begin{itemize}
	\item Upon visiting locations $3$, $4$ and $5$, take the spare tire.
	\item Follow the path described in \cref{fig:tireworld-policy}.
\end{itemize}
In contrast, a na\"ive approach to create a BDD would need more Boolean variables to represent the location $l$ through bit-blasting, consequently resulting into a larger BDD.
Indeed, without the help of the predicates derived from the DT,
we obtained a BDD with $38$ inner nodes, almost $3$ times the number of nodes of the PDD.

\begin{figure}[h]
	\centering
	\includegraphics[width=\textwidth]{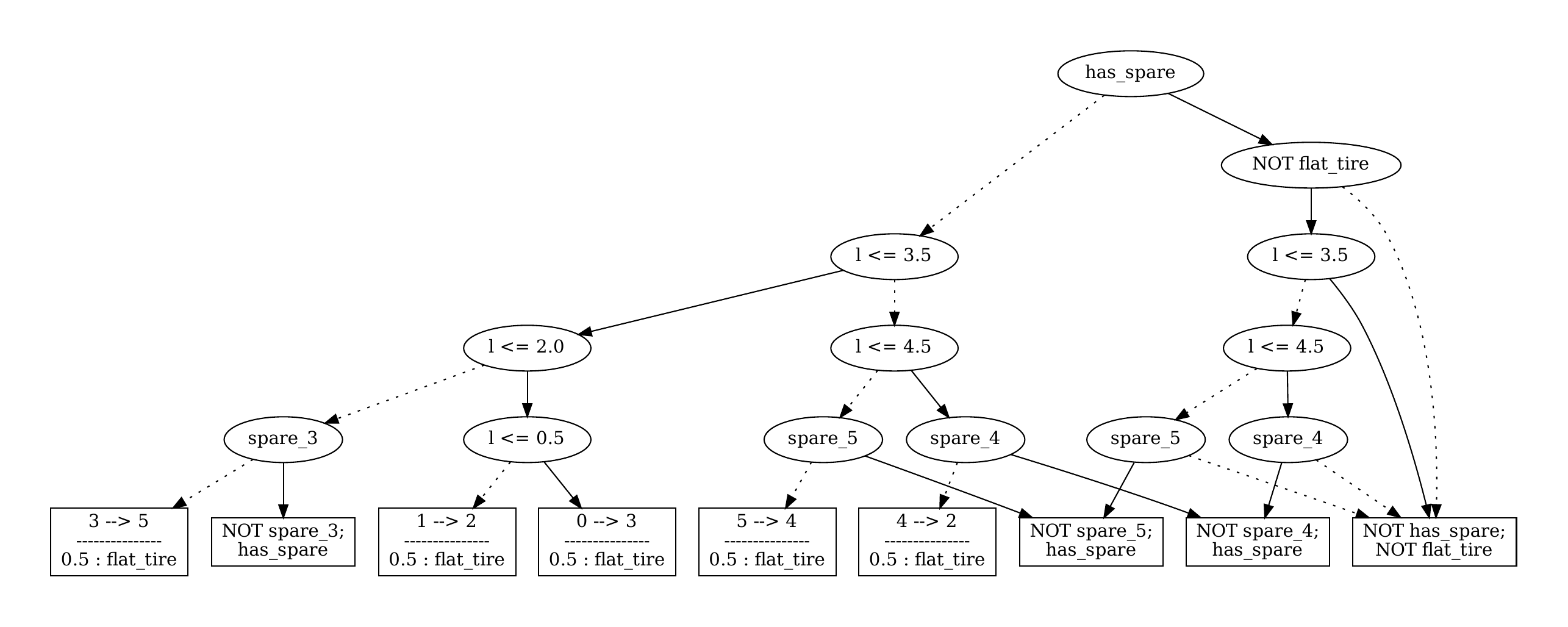}
	\caption{PDD for \emph{triangle-tireworld.9}}
	\label{fig:PDD-tireworld}
\end{figure}

\subsection{Controller to BDDs}\label{app:bdd}
\begin{table*}[h]
		\caption{Size of the BDDs produced by the two implementations. We select the minimum of  based on which we selected the value for consideration.}
		\label{table:nodes}
	\setlength{\tabcolsep}{10pt}
		\begin{tabular}{l@{\hskip 2mm}|r| r r r| r}
		\toprule
       &  & \multicolumn{3}{c|}{Controller to BDD (grouping)}  &   \\
		Controller                        & \shortstack{Used in \\ \dtcontrol \\ paper \\\cite{dtcontrol2}} & \multicolumn{1}{l}{Total} & \multicolumn{1}{l}{Inner} & {Action} & {Selected} \\
		\midrule
		triangle-tireworld.9              & 51                   & 57    & 38    & 29    & 38    \\
		pacman.5                          & 330                  & 479   & 440   & 105   & 330   \\
		rectangle-tireworld.11            & 498                  & 743   & 259   & 485   & 259   \\
		philosophers-mdp.3                & 295                  & 311   & 251   & 233   & 251   \\
		firewire\_abst.3.rounds           & 61                   & 77    & 51    & 11    & 51    \\
		rabin.3                           & 303                  & 326   & 301   & 58    & 301   \\
		ij.10                             & 436                  & 436   & 415   & 182   & 415   \\
		zeroconf.1000.4.true.correct\_max & 386                  & 566   & 520   & 71    & 386   \\
		blocksworld.5                     & 3985                 & 4414  & 4043  & 663   & 3985  \\
		cdrive.10                         & 5134                 & 8058  & 6151  & 865   & 5134  \\
		consensus.2.disagree              & 138                  & 138   & 112   & 5     & 112   \\
		beb.3-4.LineSeized                & 913                  & 1109  & 1051  & 43    & 913   \\
		csma.2-4.some\_before             & 1059                 & 1241  & 1172  & 130   & 1059  \\
		eajs.2.100.5.ExpUtil              & 1315                 & 1418  & 1349  & 113   & 1315  \\
		elevators.a-11-9                  & 6750                 & 6924  & 6790  & 70    & 6750  \\
		exploding-blocksworld.5           & 34447                & 39588 & 39436 & 367   & 34447 \\
		echoring.MaxOffline1              & 43165                & 49571 & 48765 & 1029  & 43165 \\
		wlan\_dl.0.80.deadline            & 5738                 & 6770  & 6591  & 291   & 5738  \\
		pnueli-zuck.5                     & 50128                & 59392 & 59217 & 35822 & 50128 \\
		\midrule
		cartpole                          & 312                  & 369   & 197   & 169   & 197   \\
		10rooms                           & 149                  & 1153  & 1102  & 102   & 149   \\
		helicopter                        & 1349                 & 3825  & 3348  & 598   & 1349  \\
		cruise-latest                     & 2106                 & 2125  & 2115  & 366   & 2106  \\
		dcdc                              & 577                  & 819   & 814   & 155   & 577   \\
		traffic\_30m                      & TO & 4547  & 4522  & 79    & 4522\\
\bottomrule
\end{tabular}
\end{table*}

We chose the minimum number of nodes for BDDs from two implementations.
The numbers reported in \cite{dtcontrol2} for lower part of table were obtained by using determinization,
whereas, we are interested in the permissive strategies, i.e., where more action can be taken for a state.
Hence, we executed the BDDs generation using the setting disabling determinization.
We used the \dtcontrol\ artifact for execution.

The benchmarks in the upper part of the table are determinized, so these numbers remain the same.
\cref{table:nodes} shows the results.

\subsection{Effect of Node Sharing}
\label{app:ij}
\Cref{table:ij} shows the number nodes in the PDDs representing policies for Israeli Jalfon randomized self-stabilizing protocol with increasing number of processes.
\begin{table*}[h]
	\caption{Number of nodes in PDDs representing policies for Israeli Jalfon randomized self-stabilizing protocol with increasing number of processes}
	\label{table:ij}
	\centering
	\begin{tabular}{llll}
		\toprule
		Processes & Size   & Total nodes & Shared nodes \\ 
		\toprule
		3         & 7      & 7           & 0            \\
		4         & 15     & 14          & 0            \\
		5         & 31     & 27          & 0            \\
		6         & 63     & 60          & 1            \\
		7         & 127    & 95          & 13           \\
		8         & 255    & 183         & 25           \\
		9         & 511    & 306         & 69           \\
		10        & 1023   & 517         & 222          \\
		11        & 2047   & 929         & 450          \\ 
		12        & 4095   & 1662        & 929          \\
		13        & 8191   & 3029        & 1919         \\
		14        & 16383  & 5522        & 3901         \\
		15        & 32767  & 10118       & 7547         \\
		16        & 65535  & 18797       & 14834        \\
		17        & 131071 & 34572       & 28146       \\
		\bottomrule
	\end{tabular}
\end{table*}

\subsection{Ablation Study on PDD Synthesis Pipeline}\label{app:pdd}
\cref{table:dt-learning} to \cref{table:dontcare} show the data for PDDs at each stage of our pipeline.
\cref{table:total-nodes-ablation} shows the impact on the total number of nodes (including the encoding of action sets) and \cref{table:time-ablation} shows the execution timings of each step of the pipeline.

\begin{table*}[]
	\caption{Stage 1, Decision tree learning and PDD}
	\label{table:dt-learning}
		\begin{tabular}{l|rrrrrr|S[table-format=-1.2,table-align-text-post=false]}
			\toprule
			& \multicolumn{6}{c|}{Decision Trees}  &\\
			\midrule
			Controller                        & \shortstack{Total \\ Nodes} & \shortstack{Inner \\ Nodes} & \shortstack{Leaf \\ (Action) \\ Nodes} & \shortstack{Unique \\ Actions} & \shortstack{Total \\ Predicates} & \shortstack{Unique \\ Predicates} & \shortstack{Construction \\ Time (s)} \\
			\midrule
10rooms  & 17297 & 8648 & 8649 & 9 & 8648 & 18 & 10.310441 \\
beb.3-4.LineSeized  & 65 & 32 & 33 & 29 & 32 & 22 & 0.203323 \\
blocksworld.5  & 1235 & 617 & 618 & 184 & 617 & 47 & 0.691091 \\
cartpole  & 253 & 126 & 127 & 81 & 126 & 36 & 0.500276 \\
cdrive.10  & 2401 & 1200 & 1201 & 952 & 1200 & 106 & 11.694312 \\
consensus.2.disagree  & 67 & 33 & 34 & 13 & 33 & 14 & 0.201308 \\
cruise-latest  & 987 & 493 & 494 & 3 & 493 & 169 & 17.657515 \\
csma.2-4.some{\textunderscore}before  & 103 & 51 & 52 & 33 & 51 & 25 & 0.458598 \\
dcdc  & 271 & 135 & 136 & 2 & 135 & 122 & 100.461191 \\
eajs.2.100.5.ExpUtil  & 167 & 83 & 84 & 33 & 83 & 50 & 0.864838 \\
echoring.MaxOffline1  & 1869 & 934 & 935 & 401 & 934 & 169 & 22.921879 \\
elevators.a-11-9  & 16327 & 8163 & 8164 & 65 & 8163 & 36 & 9.485567 \\
exploding-blocksworld.5  & 4981 & 2490 & 2491 & 75 & 2490 & 56 & 7.243998 \\
firewire{\textunderscore}abst.3.rounds  & 25 & 12 & 13 & 13 & 12 & 10 & 0.119658 \\
helicopter  & 6339 & 3169 & 3170 & 15 & 3169 & 81 & 51.829528 \\
ij.10  & 1291 & 645 & 646 & 10 & 645 & 10 & 0.429615 \\
pacman.5  & 43 & 21 & 22 & 19 & 21 & 10 & 0.042357 \\
philosophers-mdp.3  & 391 & 195 & 196 & 30 & 195 & 26 & 0.150762 \\
pnueli-zuck.5  & 171371 & 85685 & 85686 & 87 & 85685 & 78 & 61.602507 \\
rabin.3  & 111 & 55 & 56 & 12 & 55 & 29 & 0.120663 \\
rectangle-tireworld.11  & 481 & 240 & 241 & 241 & 240 & 21 & 0.187667 \\
traffic{\textunderscore}30m  & 12573 & 6286 & 6287 & 8 & 6286 & 58 & 3259.388619 \\
triangle-tireworld.9  & 27 & 13 & 14 & 9 & 13 & 9 & 0.021218 \\
wlan{\textunderscore}dl.0.80.deadline  & 3369 & 1684 & 1685 & 88 & 1684 & 141 & 16.986641 \\
zeroconf.1000.4.true.correct{\textunderscore}max  & 83 & 41 & 42 & 23 & 41 & 19 & 0.164252 \\

			\bottomrule 
		\end{tabular}
\end{table*}
\begin{table*}[]
	\caption{Stage 2: BDD compilation}\label{table:dd2pdd}
		\begin{tabular}{l|rr|rrrrr|S[table-format=-3.2,table-align-text-post=false]}
			\toprule
			& \multicolumn{2}{c|}{Decision Tree} & \multicolumn{5}{c|}{After BDD Compilation}                                                   &                \\
			\midrule
			Controller                        & \shortstack{Total \\ Nodes}     &  \shortstack{Decision \\ Nodes}   & \shortstack{Total \\ Nodes} & \shortstack{Decision \\ Nodes}  & \shortstack{Shared \\ Nodes}  & \shortstack{Action \\ Nodes}  &  \shortstack{Collapsed \\ Nodes} & {Time (s)} \\
			\midrule
			10rooms & 17297 & 8648 & 1451 & 1332 & 1054 & 119 & 25 & 2.337088 \\
			beb.3-4.LineSeized & 65 & 32 & 496 & 33 & 0 & 463 & 29 & 1.999578 \\
			blocksworld.5 & 1235 & 617 & 19849 & 2646 & 1437 & 17203 & 184 & 3.192066 \\
			cartpole & 253 & 126 & 3831 & 206 & 51 & 3625 & 85 & 2.086475 \\
			cdrive.10 & 2401 & 1200 & 464021 & 9442 & 1011 & 454579 & 952 & 165.191324 \\
			consensus.2.disagree & 67 & 33 & 151 & 48 & 13 & 103 & 13 & 2.037664 \\
			cruise-latest & 987 & 493 & 1102 & 1091 & 512 & 11 & 5 & 2.287071 \\
			csma.2-4.some{\textunderscore}before & 103 & 51 & 671 & 78 & 21 & 593 & 33 & 2.022658 \\
			dcdc & 271 & 135 & 154 & 149 & 14 & 5 & 3 & 1.626269 \\
			eajs.2.100.5.ExpUtil & 167 & 83 & 701 & 108 & 17 & 593 & 33 & 2.135013 \\
			echoring.MaxOffline1 & 1869 & 934 & 85430 & 4429 & 2666 & 81001 & 401 & 10.692982 \\
			elevators.a-11-9 & 16327 & 8163 & 18772 & 16563 & 8327 & 2209 & 65 & 6.24542 \\
			exploding-blocksworld.5 & 4981 & 2490 & 17781 & 14857 & 11241 & 2924 & 75 & 3.976191 \\
			firewire{\textunderscore}abst.3.rounds & 25 & 12 & 115 & 12 & 0 & 103 & 13 & 2.023417 \\
			helicopter & 6339 & 3169 & 11511 & 10294 & 5860 & 1217 & 238 & 2.945867 \\
			ij.10 & 1291 & 645 & 522 & 458 & 220 & 64 & 10 & 2.040617 \\
			pacman.5 & 43 & 21 & 236 & 28 & 0 & 208 & 19 & 1.285176 \\
			philosophers-mdp.3 & 391 & 195 & 804 & 310 & 98 & 494 & 30 & 2.041269 \\
			pnueli-zuck.5 & 171371 & 85685 & 1308316 & 1304402 & 1216526 & 3914 & 87 & 32668.696146 \\
			rabin.3 & 111 & 55 & 169 & 80 & 14 & 89 & 12 & 2.069549 \\
			rectangle-tireworld.11 & 481 & 240 & 29682 & 281 & 30 & 29401 & 241 & 4.012525 \\
			traffic{\textunderscore}30m & 12573 & 6286 & 11145 & 11088 & 9214 & 57 & 12 & 3.222716 \\
			triangle-tireworld.9 & 27 & 13 & 67 & 14 & 1 & 53 & 9 & 1.284613 \\
			wlan{\textunderscore}dl.0.80.deadline & 3369 & 1684 & 17766 & 13763 & 10194 & 4003 & 88 & 3.417063 \\
			zeroconf.1000.4.true.correct{\textunderscore}max & 83 & 41 & 358 & 60 & 9 & 298 & 23 & 2.045348 \\

			\bottomrule     
		\end{tabular}
\end{table*}
\begin{table*}[]
	\caption{Stage 3: Exhaustive Reordering}\label{table:reordering}
		\begin{tabular}{l|rr|rrrrr|S[table-format=-3.2,table-align-text-post=false]}
			\toprule
			& \multicolumn{2}{c|}{Previous Stage} & \multicolumn{5}{c|}{After Exhaustive Reordering}                                                   &                \\
			\midrule
			Controller                        & \shortstack{Total \\ Nodes}     &  \shortstack{Decision \\ Nodes}   & \shortstack{Total \\ Nodes} & \shortstack{Decision \\ Nodes}  & \shortstack{Shared \\ Nodes}  & \shortstack{Action \\ Nodes}  &  \shortstack{Collapsed \\ Nodes} & {Time (s)} \\
			\midrule
			10rooms & 1451 & 1332 & 538 & 419 & 273 & 119 & 25 & 86.729078 \\
			beb.3-4.LineSeized & 496 & 33 & 495 & 32 & 0 & 463 & 29 & 104.814762 \\
			blocksworld.5 & 19849 & 2646 & 18945 & 1742 & 750 & 17203 & 184 & 474.34582 \\
			cartpole & 3831 & 206 & 2811 & 172 & 28 & 2639 & 85 & 846.171106 \\
			cdrive.10 & 464021 & 9442 & 458407 & 3828 & 396 & 454579 & 952 & 1442.507889 \\
			consensus.2.disagree & 151 & 48 & 139 & 36 & 2 & 103 & 13 & 79.402642 \\
			cruise-latest & 1102 & 1091 & 564 & 554 & 146 & 10 & 5 & 1018.051323 \\
			csma.2-4.some{\textunderscore}before & 671 & 78 & 653 & 60 & 7 & 593 & 33 & 119.249228 \\
			dcdc & 154 & 149 & 140 & 135 & 0 & 5 & 3 & 366.765855 \\
			eajs.2.100.5.ExpUtil & 701 & 108 & 674 & 81 & 7 & 593 & 33 & 192.267123 \\
			echoring.MaxOffline1 & 85430 & 4429 & 82289 & 1288 & 362 & 81001 & 401 & 714.110685 \\
			elevators.a-11-9 & 18772 & 16563 & 8335 & 6126 & 3321 & 2209 & 65 & 147.226871 \\
			exploding-blocksworld.5 & 17781 & 14857 & 9572 & 6648 & 4270 & 2924 & 75 & 173.124975 \\
			firewire{\textunderscore}abst.3.rounds & 115 & 12 & 115 & 12 & 0 & 103 & 13 & 38.529454 \\
			helicopter & 11511 & 10294 & 6718 & 5577 & 2456 & 1141 & 238 & 732.725871 \\
			ij.10 & 522 & 458 & 516 & 452 & 225 & 64 & 10 & 41.959595 \\
			pacman.5 & 236 & 28 & 231 & 23 & 1 & 208 & 19 & 53.081997 \\
			philosophers-mdp.3 & 804 & 310 & 740 & 246 & 64 & 494 & 30 & 109.693295 \\
			pnueli-zuck.5 & 1308316 & 1304402 & 499947 & 496033 & 423036 & 3914 & 87 & 9014.644852 \\
			rabin.3 & 169 & 80 & 151 & 62 & 4 & 89 & 12 & 118.187025 \\
			rectangle-tireworld.11 & 29682 & 281 & 29675 & 274 & 28 & 29401 & 241 & 622.209227 \\
			traffic{\textunderscore}30m & 11145 & 11088 & 5514 & 5461 & 4291 & 53 & 12 & 95.1932 \\
			triangle-tireworld.9 & 67 & 14 & 67 & 14 & 1 & 53 & 9 & 24.668052 \\
			wlan{\textunderscore}dl.0.80.deadline & 17766 & 13763 & 7516 & 3513 & 1701 & 4003 & 88 & 1122.084674 \\
			zeroconf.1000.4.true.correct{\textunderscore}max & 358 & 60 & 353 & 55 & 7 & 298 & 23 & 83.612321 \\

			\bottomrule     
		\end{tabular}
\end{table*}
\begin{table*}[]
	\caption{Stage 4: Consistency Transformation}\label{table:consistency}
		\begin{tabular}{l|rr|rrrrr|S[table-format=-3.2,table-align-text-post=false]}
			\toprule
			& \multicolumn{2}{c|}{Previous Stage} & \multicolumn{5}{c|}{After Consistency Transformation}                                                   &                \\
			\midrule
			Controller                        & \shortstack{Total \\ Nodes}     &  \shortstack{Decision \\ Nodes}   & \shortstack{Total \\ Nodes} & \shortstack{Decision \\ Nodes}  & \shortstack{Shared \\ Nodes}  & \shortstack{Action \\ Nodes}  &  \shortstack{Collapsed \\ Nodes} & {Time (s)} \\
			\midrule
			10rooms & 538 & 419 & 463 & 344 & 211 & 119 & 25 & 0.225848 \\
			beb.3-4.LineSeized & 495 & 32 & 495 & 32 & 0 & 463 & 29 & 0.003077 \\
			blocksworld.5 & 18945 & 1742 & 18729 & 1526 & 544 & 17203 & 184 & 1.770755 \\
			cartpole & 2811 & 172 & 2772 & 133 & 0 & 2639 & 85 & 0.031675 \\
			cdrive.10 & 458407 & 3828 & 458407 & 3828 & 396 & 454579 & 952 & 14.241887 \\
			consensus.2.disagree & 139 & 36 & 136 & 33 & 1 & 103 & 13 & 0.001497 \\
			cruise-latest & 564 & 554 & 489 & 479 & 67 & 10 & 5 & 0.009812 \\
			csma.2-4.some{\textunderscore}before & 653 & 60 & 651 & 58 & 4 & 593 & 33 & 0.006999 \\
			dcdc & 140 & 135 & 140 & 135 & 0 & 5 & 3 & 0.001922 \\
			eajs.2.100.5.ExpUtil & 674 & 81 & 674 & 81 & 6 & 593 & 33 & 0.012782 \\
			echoring.MaxOffline1 & 82289 & 1288 & 82275 & 1274 & 331 & 81001 & 401 & 5.246414 \\
			elevators.a-11-9 & 8335 & 6126 & 8286 & 6077 & 3283 & 2209 & 65 & 5.499225 \\
			exploding-blocksworld.5 & 9572 & 6648 & 9428 & 6504 & 4037 & 2924 & 75 & 9.828777 \\
			firewire{\textunderscore}abst.3.rounds & 115 & 12 & 115 & 12 & 0 & 103 & 13 & 0.000734 \\
			helicopter & 6718 & 5577 & 4417 & 3276 & 546 & 1141 & 238 & 0.243228 \\
			ij.10 & 516 & 452 & 516 & 452 & 225 & 64 & 10 & 0.01987 \\
			pacman.5 & 231 & 23 & 231 & 23 & 0 & 208 & 19 & 0.001599 \\
			philosophers-mdp.3 & 740 & 246 & 699 & 205 & 6 & 494 & 30 & 0.017889 \\
			pnueli-zuck.5 & 499947 & 496033 & 77053 & 73139 & 27431 & 3914 & 87 & 85.846603 \\
			rabin.3 & 151 & 62 & 148 & 59 & 1 & 89 & 12 & 0.002501 \\
			rectangle-tireworld.11 & 29675 & 274 & 29641 & 240 & 0 & 29401 & 241 & 0.177951 \\
			traffic{\textunderscore}30m & 5514 & 5461 & 2550 & 2497 & 1641 & 53 & 12 & 0.306678 \\
			triangle-tireworld.9 & 67 & 14 & 66 & 13 & 0 & 53 & 9 & 0.0006 \\
			wlan{\textunderscore}dl.0.80.deadline & 7516 & 3513 & 6007 & 2004 & 354 & 4003 & 88 & 2.18757 \\
			zeroconf.1000.4.true.correct{\textunderscore}max & 353 & 55 & 354 & 56 & 5 & 298 & 23 & 0.00434 \\
			
			\bottomrule     
		\end{tabular}
\end{table*}
\begin{table*}[]
	\caption{Stage 5: Care-set Reduction}\label{table:dontcare}
		\begin{tabular}{l|rr|rrrrr|S[table-format=-3.2,table-align-text-post=false]}
			\toprule
			& \multicolumn{2}{c|}{Previous Stage} & \multicolumn{5}{c|}{After Care-set Reduction}                                                   &                \\
			\midrule
			Controller                        & \shortstack{Total \\ Nodes}     &  \shortstack{Decision \\ Nodes}   & \shortstack{Total \\ Nodes} & \shortstack{Decision \\ Nodes}  & \shortstack{Shared \\ Nodes}  & \shortstack{Action \\ Nodes}  &  \shortstack{Collapsed \\ Nodes} & {Time (s)} \\
			\midrule
			10rooms & 463 & 344 & 463 & 344 & 211 & 119 & 25 & 0.576916 \\
			beb.3-4.LineSeized & 495 & 32 & 495 & 32 & 0 & 463 & 29 & 0.091822 \\
			blocksworld.5 & 18729 & 1526 & 17999 & 796 & 13 & 17203 & 184 & 0.369478 \\
			cartpole & 2772 & 133 & 2765 & 126 & 0 & 2639 & 85 & 0.039679 \\
			cdrive.10 & 458407 & 3828 & 455779 & 1200 & 0 & 454579 & 952 & 0.263466 \\
			consensus.2.disagree & 136 & 33 & 134 & 31 & 1 & 103 & 13 & 0.021182 \\
			cruise-latest & 489 & 479 & 486 & 476 & 66 & 10 & 5 & 119.236493 \\
			csma.2-4.some{\textunderscore}before & 651 & 58 & 646 & 53 & 2 & 593 & 33 & 0.126468 \\
			dcdc & 140 & 135 & 140 & 135 & 0 & 5 & 3 & 44.44782 \\
			eajs.2.100.5.ExpUtil & 674 & 81 & 678 & 85 & 7 & 593 & 33 & 0.714251 \\
			echoring.MaxOffline1 & 82275 & 1274 & 81971 & 970 & 94 & 81001 & 401 & 141.428935 \\
			elevators.a-11-9 & 8286 & 6077 & 7764 & 5555 & 2756 & 2209 & 65 & 1.273282 \\
			exploding-blocksworld.5 & 9428 & 6504 & 6559 & 3635 & 1346 & 2924 & 75 & 44.953039 \\
			firewire{\textunderscore}abst.3.rounds & 115 & 12 & 115 & 12 & 0 & 103 & 13 & 0.004293 \\
			helicopter & 4417 & 3276 & 4299 & 3158 & 486 & 1141 & 238 & 26.675852 \\
			ij.10 & 516 & 452 & 517 & 453 & 222 & 64 & 10 & 0.011955 \\
			pacman.5 & 231 & 23 & 229 & 21 & 0 & 208 & 19 & 0.002571 \\
			philosophers-mdp.3 & 699 & 205 & 697 & 203 & 5 & 494 & 30 & 0.026086 \\
			pnueli-zuck.5 & 77053 & 73139 & 76106 & 72192 & 26941 & 3914 & 87 & 160.173035 \\
			rabin.3 & 148 & 59 & 147 & 58 & 0 & 89 & 12 & 0.061869 \\
			rectangle-tireworld.11 & 29641 & 240 & 29641 & 240 & 0 & 29401 & 241 & 0.011494 \\
			traffic{\textunderscore}30m & 2550 & 2497 & 2403 & 2350 & 1538 & 53 & 12 & 1281.846615 \\
			triangle-tireworld.9 & 66 & 13 & 66 & 13 & 0 & 53 & 9 & 0.000951 \\
			wlan{\textunderscore}dl.0.80.deadline & 6007 & 2004 & 5827 & 1824 & 199 & 4003 & 88 & 211.476494 \\
			zeroconf.1000.4.true.correct{\textunderscore}max & 354 & 56 & 342 & 44 & 1 & 298 & 23 & 0.018663 \\

			\bottomrule     
		\end{tabular}
\end{table*}

\begin{table*}[h!]
	\centering
	\caption{Total number of nodes in the PDDs after each step of the pipeline}
	\setlength{\tabcolsep}{7pt}
	\begin{adjustbox}{width=\textwidth,center}
			\begin{tabular}{l|rrrrr}
				\toprule
				Controllers              & DT learning & BDD compilation & Reordering & Consistency transformation & Care-set reduction\\
					\toprule
		10rooms	& 17297	& 1451	& 538	& 463	& 463\\
		cartpole	& 253	& 3831	& 2811	& 2772	& 2765\\
		cruise-latest	& 987	& 1102	& 564	& 489	& 486\\
		dcdc	& 271	& 154	& 140	& 140	& 140\\
		helicopter	& 6339	& 11511	& 6718	& 4417	& 4299\\
		traffic\_30m	& 12573	& 11145	& 5514	& 2550	& 2403\\
		\midrule
		beb.3-4.LineSeized	& 65	& 496	& 495	& 495	& 495\\
		blocksworld.5	& 1235	& 19849	& 18945	& 18729	& 17999\\
		cdrive.10	& 2401	& 464021	& 458407	& 458407	& 455779\\
		consensus.2.disagree	& 67	& 151	& 139	& 136	& 134\\
		csma.2-4.some\_before	& 103	& 671	& 653	& 651	& 646\\
		eajs.2.100.5.ExpUtil	& 167	& 701	& 674	& 674	& 678\\
		echoring.MaxOffline1	& 1869	& 85430	& 82289	& 82275	& 81971\\
		elevators.a-11-9	& 16327	& 18772	& 8335	& 8286	& 7764\\
		exploding-blocksworld.5	& 4981	& 17781	& 9572	& 9428	& 6559\\
		firewire\_abst.3.rounds	& 25	& 115	& 115	& 115	& 115\\
		ij.10	& 1291	& 522	& 516	& 516	& 517\\
		pacman.5	& 43	& 236	& 231	& 231	& 229\\
		philosophers-mdp.3	& 391	& 804	& 740	& 699	& 697\\
		pnueli-zuck.5	& 171371	& 1308316	& 499947	& 77053	& 76106\\
		rabin.3	& 111	& 169	& 151	& 148	& 147\\
		rectangle-tireworld.11	& 481	& 29682	& 29675	& 29641	& 29641\\
		triangle-tireworld.9	& 27	& 67	& 67	& 66	& 66\\
		wlan\_dl.0.80.deadline	& 3369	& 17766	& 7516	& 6007	& 5827\\
		zeroconf.1000.4.true.correct\_max	& 83	& 358	& 353	& 354	& 342\\			
		\bottomrule    
				\end{tabular}
		\end{adjustbox}
		\label{table:total-nodes-ablation}
	\end{table*}

\begin{table*}[h!]
	\centering
	\caption{Time taken (in seconds) in each steps of the pipeline.}
	\setlength{\tabcolsep}{7pt}
	\begin{adjustbox}{width=\textwidth,center}
			\begin{tabular}{l|*{5}{S[table-format=3.2, round-integer-to-decimal, round-mode = places, round-precision = 2]}}
				\toprule
				Controllers              & {DT learning} & {BDD compilation} & {Reordering} & {Consistency transformation} & {Care-set reduction}\\
				\toprule
				10rooms	& 10.310441	& 2.337088	& 86.729078	& 0.225848	& 0.576916\\
				cartpole	& 0.500276	& 2.086475	& 846.171106	& 0.031675	& 0.039679\\
				cruise-latest	& 17.657515	& 2.287071	& 1018.051323	& 0.009812	& 119.236493\\
				dcdc	& 100.461191	& 1.626269	& 366.765855	& 0.001922	& 44.44782\\
				helicopter	& 51.829528	& 2.945867	& 732.725871	& 0.243228	& 26.675852\\
				traffic\_30m	& 3259.388619	& 3.222716	& 95.1932	& 0.306678	& 1281.846615\\
				\midrule
				beb.3-4.LineSeized	& 0.203323	& 1.999578	& 104.814762	& 0.003077	& 0.091822\\
				blocksworld.5	& 0.691091	& 3.192066	& 474.34582	& 1.770755	& 0.369478\\
				cdrive.10	& 11.694312	& 165.191324	& 1442.507889	& 14.241887	& 0.263466\\
				consensus.2.disagree	& 0.201308	& 2.037664	& 79.402642	& 0.001497	& 0.021182\\
				csma.2-4.some\_before	& 0.458598	& 2.022658	& 119.249228	& 0.006999	& 0.126468\\
				eajs.2.100.5.ExpUtil	& 0.864838	& 2.135013	& 192.267123	& 0.012782	& 0.714251\\
				echoring.MaxOffline1	& 22.921879	& 10.692982	& 714.110685	& 5.246414	& 141.428935\\
				elevators.a-11-9	& 9.485567	& 6.24542	& 147.226871	& 5.499225	& 1.273282\\
				exploding-blocksworld.5	& 7.243998	& 3.976191	& 173.124975	& 9.828777	& 44.953039\\
				firewire\_abst.3.rounds	& 0.119658	& 2.023417	& 38.529454	& 0.000734	& 0.004293\\
				ij.10	& 0.429615	& 2.040617	& 41.959595	& 0.01987	& 0.011955\\
				pacman.5	& 0.042357	& 1.285176	& 53.081997	& 0.001599	& 0.002571\\
				philosophers-mdp.3	& 0.150762	& 2.041269	& 109.693295	& 0.017889	& 0.026086\\
				pnueli-zuck.5	& 61.602507	& 32668.69615	& 9014.644852	& 85.846603	& 160.173035\\
				rabin.3	& 0.120663	& 2.069549	& 118.187025	& 0.002501	& 0.061869\\
				rectangle-tireworld.11	& 0.187667	& 4.012525	& 622.209227	& 0.177951	& 0.011494\\
				triangle-tireworld.9	& 0.021218	& 1.284613	& 24.668052	& 0.0006	& 0.000951\\
				wlan\_dl.0.80.deadline	& 16.986641	& 3.417063	& 1122.084674	& 2.18757	& 211.476494\\
				zeroconf.1000.4.true.correct\_max	& 0.164252	& 2.045348	& 83.612321	& 0.00434	& 0.018663\\
				\midrule
				Average & 142.949513	& 1316.036564	& 712.8542768	& 5.02760932	& 81.35410836\\
				\bottomrule
			\end{tabular}
		\end{adjustbox}
		\label{table:time-ablation}
	\end{table*}
\clearpage
\end{document}